\documentclass{article} 
\usepackage{iclr2023_conference,times}

\usepackage{amsmath,amsfonts,bm}

\def\eqref#1{equation~\ref{#1}}

\def\1{\bm{1}}

\DeclareMathAlphabet{\mathsfit}{\encodingdefault}{\sfdefault}{m}{sl}
\SetMathAlphabet{\mathsfit}{bold}{\encodingdefault}{\sfdefault}{bx}{n}

\newcommand{\R}{\mathbb{R}}

\usepackage[hidelinks]{hyperref}
\usepackage{url}

\usepackage[utf8]{inputenc} 
\usepackage[T1]{fontenc}    
\usepackage{booktabs}       
\usepackage{amsfonts}       
\usepackage{nicefrac}       
\usepackage{microtype}      
\usepackage{xcolor}         

\usepackage{graphicx}
\usepackage{amsmath,bm,bbm}
\usepackage{amsfonts}
\usepackage{xcolor}         
\usepackage{amsthm}        
\usepackage[algoruled,lined,linesnumbered]{algorithm2e} 
\usepackage[capitalize,nameinlink,noabbrev]{cleveref}
\usepackage[skip=0pt,belowskip=-2pt]{caption}
\usepackage{subcaption}
\usepackage{chngcntr}
\usepackage{dblfloatfix}    
\usepackage{wrapfig}

\newtheorem{theorem}{Theorem}
\newtheorem{proposition}{Proposition}
\newtheorem{assumption}{Assumption}
\newtheorem{lemma}{Lemma}

\newcommand{\mdp}{\calM}
\newcommand{\trans}{p}

\newcommand{\sspace}{\mathcal{S}}
\newcommand{\aspace}{\mathcal{A}}
\newcommand{\rspace}{\mathcal{R}}

\newcommand{\setpolicies}{\Pi}

\newcommand{\optimalr}{r_*}

\newcommand{\smdp}{\hat \calM}
\newcommand{\otrans}{\hat p}
\newcommand{\ospace}{\mathcal{O}}
\newcommand{\lspace}{ \mathcal{L}}
\newcommand{\orspace}{ \mathcal{\hat R}}
\newcommand{\optionr}{r}
\newcommand{\optionl}{l}
\newcommand{\opi}{\pi}
\newcommand{\opistar}{\opi}
\newcommand{\otransmatrix}{P_{\opi}}

\newcommand{\olimitingmatrix}{P_{\opi}^\infty}
\newcommand{\olimitingmatrixstar}{P_{\opi^*}^\infty}
\newcommand{\ofundmatrix}{Z_{\opi}}
\newcommand{\ooptimalr}{\hat r_*}
\newcommand{\oonestepr}{r_{\opi}}
\newcommand{\oonestepl}{l_{\opi}}
\newcommand{\ostationarydist}{d_{\opi}}
\newcommand{\osetpolicies}{\hat {\Pi}}
\newcommand{\osetoptimalpolicies}{\hat \Pi_{*}}
\newcommand{\ooptimalrecurrentstates}{{R}_{*}}
\newcommand{\orecurrentstates}{R_{\opi}}
\newcommand{\ooptimalnbrecurrentclasses}{n_*}

\newcommand{\ispace}{\mathcal{I}}

\newcommand{\GRVIQsolutionq}{\mathcal{Q}_\#}
\newcommand{\GRVIQsolutionrbar}{r_\#}
\newcommand{\GDiffQsolutionq}{\mathcal{Q}_\infty}

\newcommand{\bbE}{\mathbb{E}}
\newcommand{\bbR}{\mathbb{R}}

\newcommand{\calS}{\mathcal{S}}

\newcommand{\calO}{\mathcal{O}}
\newcommand{\calI}{\mathcal{I}}

\newcommand{\calF}{\mathcal{F}}

\newcommand{\calM}{\mathcal{M}}

\newcommand{\calQ}{\mathcal{Q}}
\newcommand{\calV}{\mathcal{V}}

\newcommand{\bfI}{\mathbf{I}}

\newcommand{\norm}[1]{\left\lVert#1\right\rVert}
\newcommand{\abs}[1]{\left\lvert#1\right\rvert}
\newcommand{\cardS}{{\abs{\calS}}}

\newcommand{\cardO}{{\abs{\calO}}}
\newcommand{\cardI}{{\abs{\calI}}}

\usepackage{xargs}      

\usepackage{xspace}

\title{On Convergence of Average-Reward Off-Policy Control Algorithms in Weakly Communicating MDPs}

\author{%
  Yi Wan\\
  University of Alberta\\
  Edmonton, Canada\\
  \texttt{\{wan6\}@ualberta.ca}
  \And
  Richard S. Sutton\\
  University of Alberta, DeepMind\\
  Edmonton, Canada\\
  \texttt{\{rsutton\}@ualberta.ca}
}

\iclrfinalcopy 
\begin{document}

\maketitle

\begin{abstract}
We show two average-reward off-policy control algorithms, Differential Q-learning (Wan, Naik, \& Sutton 2021a) and RVI Q-learning (Abounadi Bertsekas \& Borkar 2001), converge in weakly communicating MDPs. Weakly communicating MDPs are the most general MDPs that can be solved by a learning algorithm with a single stream of experience. The original convergence proofs of the two algorithms require that the solution set of the average-reward optimality equation only has one degree of freedom, which is not necessarily true for weakly communicating MDPs. To the best of our knowledge, our results are the first showing average-reward off-policy control algorithms converge in weakly communicating MDPs. As a direct extension, we show that average-reward options algorithms for temporal abstraction introduced by Wan, Naik, \& Sutton (2021b) converge if the Semi-MDP induced by options is weakly communicating. 
\end{abstract}

\section{Introduction}
\label{sec: Introduction}
Modern reinforcement learning algorithms are designed to maximize the agent's goal in either the episodic setting or the continuing setting. In both settings, there is an agent continually interacting with its world, which is usually assumed to be a Markov Decision Process (MDP). For episodic problems, there is a special terminal state and a set of start states. If the agent reaches the terminal state, it will be reset to one of the start states. Continuing problems are different in that there is no terminal state, and the agent will never be reset by the world. For continuing problems, two commonly considered objectives are the discounted objective and the average-reward objective. The discount factor in the discounted objective has been observed to be deprecated in the function approximation control setting, suggesting that the average-reward objective might be more suitable for continuing problems.

In this paper, we extend the convergence results of two off-policy control algorithms for the average-reward objective from a sub-class of MDPs to the most general class of MDPs that could be solved by algorithms learning from a single stream of experience. These algorithms learn a policy that achieves the best possible average-reward rate, using data generated by some other policy that the agent may not have control of.  Designing convergent off-policy algorithms for the average-reward objective is challenging. While there are several off-policy learning algorithms in the literature, the only known convergent algorithms are SSP Q-learning and RVI Q-learning, both by Abounadi, Bertsekas, \& Borkar (2001), the algorithm by Ren \& Krogh (2001), and Differential Q-learning by Wan, Naik, \& Sutton (2021a). Others either do not have a convergence theory (Schwartz 1993, Singh 1994; Bertsekas \& Tsitsiklis 1996, Das 1999) or have incorrect proof (Yang 2016, Gosavi 2004). \footnote{See Appendix D in Wan et al.\ (2021a) for a discussion about Yang's proof and see \cref{sec: Gosavi incorrect proof} of this paper for a discussion about Gosavi's proof.} 

The algorithm by Ren \& Krogh (2001) requires knowledge of properties of the MDP which are not typically known. The convergence of SSP Q-learning requires knowing a state that is recurrent under all policies. The convergence of the RVI Q-learning algorithm (Abounadi et al. 2001) was developed for unichain MDPs, which just means that the Markov chain induced by any stationary policy is unichain \footnote{A Markov chain is unichain if there is only one recurrent class in the Markov chain, plus a possibly empty set of transient states.}. The convergence of Differential Q-learning (Wan et al. 2021a) requires a weaker assumption -- the solution set of the average-reward optimality equation (formally defined later in \eqref{eq: action-value optimality equation}) only has one degree of freedom (all the solutions are different by a constant vector). This assumption can be satisfied if, for example, all optimal policies are unichain. It is clear that RVI Q-learning also converges under this assumption. 

It is not rare that the solution set of the average-reward optimality equation has more than one degree of freedom (e.g., the MDP at the bottom of \cref{fig:different mdps}). In this case, the proofs of RVI Q-learning and Differential Q-learning would not go through. Technically, this is because both two proofs require that the uniqueness of the solution of the action values up to an additive constant (one degree of freedom) in the average-reward optimality equation, so that there is a unique equalibrium in the ordinary differential equations associated with the two algorithms.

\begin{wrapfigure}{r}{0.5\textwidth}
    \vspace{-5mm}
    \centering
    \includegraphics[width=0.5\textwidth]{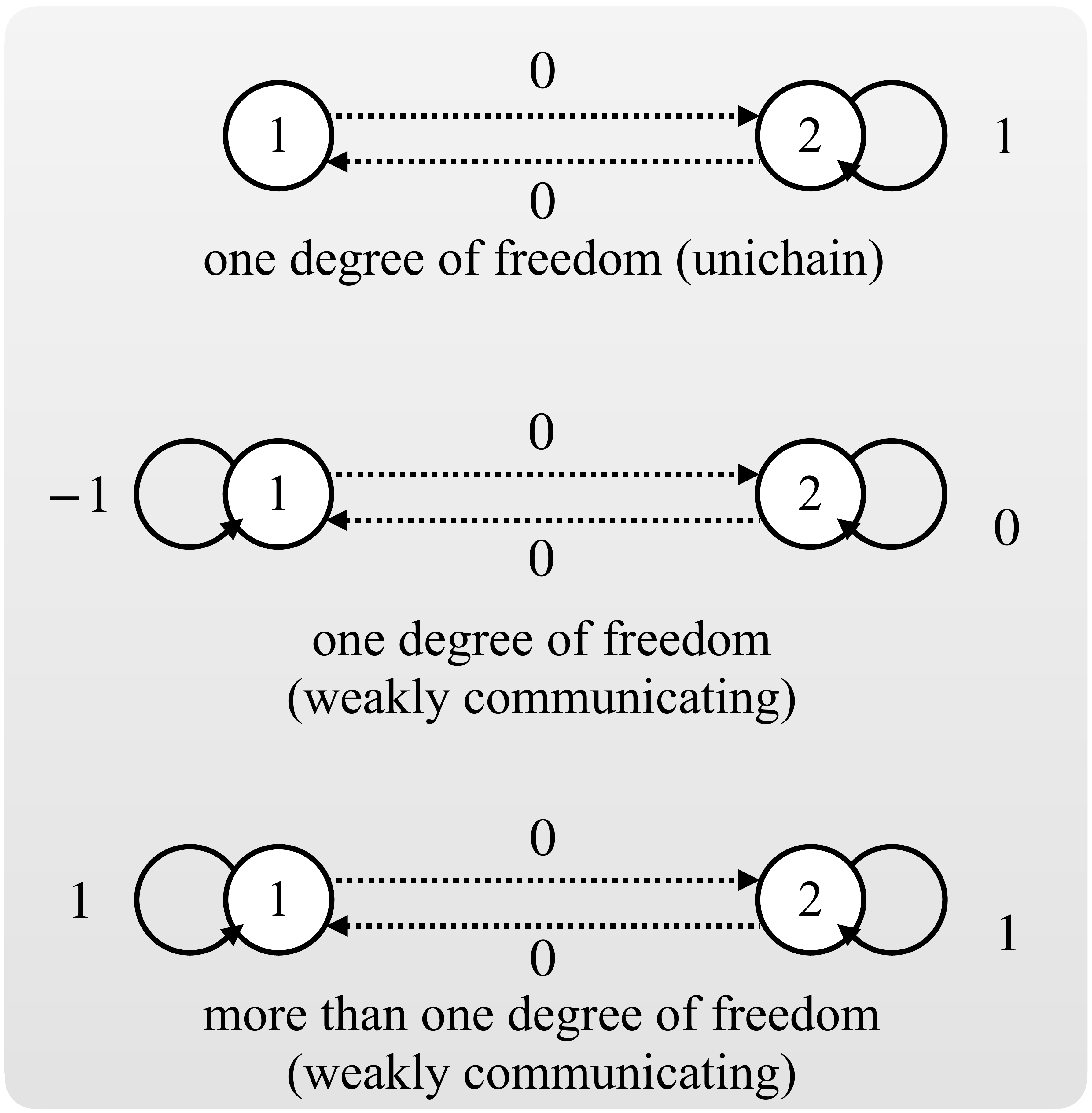}
\caption{Examples of three different types of MDPs. In each of the three MDPs, there are two states marked by two circles respectively. There are two actions \texttt{solid} and \texttt{dashed}, both causing deterministic effects. \emph{Top}: The solution set of $q$ in the average-reward optimality equation (\eqref{eq: action-value optimality equation}) is $\{q(1, \texttt{dashed}) = c-1, q(2, \texttt{solid}) = c, q(2, \texttt{\texttt{dashed}}) = c - 2, \forall c \in \bbR \}$ and has one degree of freedom. The Markov chain under every stationary policy is unichain. \emph{Middle}: The solution set of $q$ is $\{q(1, \texttt{solid}) = c-3, q(1, \texttt{dashed}) = c-1, q(2, \texttt{solid}) = c, q(2, \texttt{\texttt{dashed}}) = c - 2, \forall c \in \bbR \}$ and has one degree of freedom. The MDP is weakly-communicating. \emph{Down}: The solution set of $q$ is has more than one degree of freedom. Note that both $\{q(1, \texttt{solid}) = 0, q(1, \texttt{dashed}) = -2, q(2, \texttt{solid}) = -1, q(2, \texttt{dashed}) = -1\}$ and $\{q(1, \texttt{solid}) = 0, q(1, \texttt{dashed}) = -1, q(2, \texttt{solid}) = 0, q(2, \texttt{dashed}) = -1\}$ are solutions of $q$ in \eqref{eq: action-value optimality equation} and these two solutions are not different by a constant vector. The MDP is weakly-communicating.}
\vspace{-15mm}
\label{fig:different mdps}
\end{wrapfigure}

A more general class of MDPs, called weakly communicating MDPs, may have more than a single degree of freedom in the solution set of the associated optimality equation.
The definition of these MDPs is also natural: except for a possibly empty set of states that are transient under every policy, all states are reachable from every other state in a finite number of steps with a non-zero probability. It has been observed that the set of weakly communicating MDPs is the most general set of MDPs such that there exists a learning algorithm that can, using a single stream of experience, guarantee to identify a policy that achieves the optimal average reward rate in the MDP (Barlett \& Tewari 2009).

In this paper, we show the convergence of RVI Q-learning and Differential Q-learning in weakly communicating MDPs, without requiring any additional assumptions compared with their original convergence theories. Two key steps in our proof are 1) showing that the solution sets of the two algorithms are non-empty, closed, bounded, and connected, and 2) showing that $0$ is the unique solution for the average-reward optimality equation when all rewards are $0$. With these two results, we use asynchronous stochastic approximation results by Borkar (2009) to show convergence to the solution sets. As a direct extension of the above results, we also show the convergence of two algorithms that extend the Differential Q-learning algorithm to the options framework, introduced by Wan et al. (2021b), if the Semi-MDP induced by a given MDP and a given set of options is weakly communicating.

\section{Preliminaries}

Consider a finite Markov decision process, defined by the tuple $\mdp \doteq (\sspace, \aspace, \rspace, \trans)$, where $\sspace$ is a set of states, $\aspace$ is a set of actions, $\rspace$ is a set of rewards, and $p : \sspace \times \rspace \times \sspace \times \aspace \to [0, 1]$ is the dynamics of the MDP. 
At each time step $t$, the agent observes the state of the MDP $S_t\in\sspace$ and chooses an action $A_t\in\aspace$ using some policy $b:\aspace \times \sspace \rightarrow[0, 1]$, then receives from the environment a reward $R_{t+1}\in\rspace$ and the next state $S_{t+1}\in\sspace$, and so on. The transition dynamics are defined as $\trans(s', r \mid s, a) \doteq \Pr(S_{t+1} = s', R_{t+1} = r \mid S_t = s, A_t = a)$ for all $s, s' \in \sspace, a \in \aspace$, and $r \in \rspace$. Denote the set of stationary Markov policies $\setpolicies$. 

The reward rate of a policy $\pi$ starting from a given start state $s$ can be defined as:
\begin{align}
\label{eq:avg_rew_definition}
    r(\pi,s) \doteq \lim_{n \to \infty} \frac{1}{n} \sum_{t = 1}^n \mathbb{E}[R_t \mid S_0=s, A_{0:t-1} \sim \pi].
\end{align}
Given an arbitrary MDP, the agent may not even be able to visit all states and would therefore miss the chance of learning, for every state $s$, a policy that achieves the optimal reward rate $\sup_\pi r(\pi, s)$ and the agent can at most learn an optimal policy for a set of states, each of which is reachable from every other state. Such a set of states is often called \emph{communicating}. Formally speaking, we say a set of states \emph{communicating}, if there exists a policy such that moving from either one state in the set to the other one in the set in a finite number of steps has a positive probability. If the entire state space of an MDP is communicating, we say the MDP \emph{communicating}. \emph{Weakly communicating} MDPs generalize over communicating MDPs. In weakly communicating MDPs, in addition to a \emph{closed} communicating set of states, there is a possibly empty set of states that are transient under every policy. 



For weakly communicating MDPs, there exists a unique optimal reward rate $\optimalr$, which does not depend on the start state. We say a policy is optimal if it achieves $\optimalr$ regardless of the start state. The goal of an off-policy control algorithm is to learn an optimal policy from the stream of experience $\ldots, S_t, A_t, R_{t+1}, S_{t+1}, \ldots$ generated by a behavior policy that is not necessarily the same as the agent's learned policy. Both RVI Q-learning and Differential Q-learning achieve this goal by solving $\bar r$ and $q$ in the optimality equation:
\begin{align}
    q(s, a) &= \sum_{s', r} \trans (s', r \mid s, a) (r - \bar r + \max_{a'} q(s', a')), \quad \forall\ s \in \sspace, a \in \aspace \label{eq: action-value optimality equation}.
\end{align}
It is known that $\optimalr$ is the unique solution of $\bar r$ and any greedy policy w.r.t. any solution of $q$ is an optimal policy. In addition, shifting any solution of $q$  by any constant vector results in the other solution of $q$. Finally, unlike in unichain MDPs, where all solutions of $q$ are different by some constant vector, in weakly communicating MDPs, solutions of $q$ may have multiple degrees of freedom. That is, if $q_1, q_2$ are both solutions of $q$, it is possible that $q_1 \neq q_2 + ce, \forall\ c \in \bbR$, where $e$ denotes the all-one vector.

If the agent has a set of options, it may choose to execute these options. Each option $o$ in $\ospace$ has two components: the option's \textit{policy} $\pi^o: \aspace \times \sspace \to [0, 1]$, and the termination probability $\beta^o: \sspace \to [0, 1]$.
For simplicity, for any $s \in \sspace, o \in \ospace$, we use $\pi(a \mid s, o)$ to denote $\pi^o(a, s)$ and $\beta(s, o)$ to denote $\beta^o(s)$. If the agent executes option $o$ at state $s$, the option's policy is followed, until the option terminates.
Let $\lspace$ be the set of all possible lengths of options and $\orspace$ be the set of all possible cumulative rewards. Note that $\lspace$ and $\orspace$ are possibly countably infinite. Let $\otrans (s', r, l \mid s, o)$ be, when executing option $o$ starting from state $s$, the probability of terminating at state $s'$, with cumulative reward $r$ and length $l$. Formally, for any $s, s' \in \sspace, o \in \ospace, r \in \orspace, l \in \lspace$, $\otrans$ can be defined recursively in the following way:
\begin{align}
    & \otrans (s', r, l \mid s, o) \doteq \sum_{a} \pi(a \mid s, o) \sum_{\tilde s, \tilde r} \trans (\tilde s, \tilde r \mid s, a) \nonumber\\
    & [\beta(\tilde s, o) \bfI(\tilde s =s', \tilde r = r, \tilde l = 1) + (1 - \beta(\tilde s, o)) \otrans (s', r - \tilde r, l - 1 \mid \tilde s, o)],
\end{align}
where $\bfI$ is an indicator function.

An MDP $\mdp$ and a set of options $\ospace$ results in a Semi-MDP (SMDP) $\smdp \doteq (\sspace, \ospace, \lspace, \orspace, \otrans)$. 

Given an MDP and a set of options, if the agent chooses options using a meta policy, which is a policy that chooses from options, ${\opi}: \ospace \times \sspace \to [0, 1]$ and executes these options, we denote the sequence of option transitions by $\ldots, \hat S_n, \hat O_n, \hat R_{n+1}, \hat S_{n+1}, \ldots$. For the associated SMDP, the \textit{reward rate} of $\opi$ given a start state $s$ can be defined as $\optionr^C(\opi, s) \doteq \lim_{t \to \infty} \bbE_{\opi} [\sum_{i = 1}^t R_i \mid S_0 = s] / t$ or $\optionr(\opi, s) \doteq \lim_{n \to \infty} \bbE_{\opi} [\sum_{i = 0}^n \hat R_i \mid \hat S_0 = s] / \bbE_{\opi} [\sum_{i = 0}^n \hat L_i \mid \hat S_0 = s]$. Both limits exist and are equivalent (by Puterman's (1994) Propositions 11.4.1 and 11.4.7) under the following assumption: 

\begin{assumption}\label{assu: option assumption}
For each option $o \in \ospace$, when executing the option, there is a non-zero probability of terminating the option after at most $\cardS$ stages, regardless of the state at which this option is initiated.
\end{assumption}

\begin{proposition}\label{prop: finite expectation and variance}
Under \cref{assu: option assumption}, the expected value as well as the variance of the execution time and cumulative reward of every option at each state exist and are finite.
\end{proposition}

We say an SMDP is weakly communicating if the MDP with state space $\sspace$, action space $\ospace$, reward space $\orspace$, and transition function $\sum_{l} \otrans(s, r, l \mid s, o)$ is weakly communicating. Just as in the MDP setting, if the SMDP is weakly communicating, the optimal reward rate $\ooptimalr \doteq \sup_{\opi \in \osetpolicies} \optionr (\opi, s)$, where $\osetpolicies$ denotes the set of stationary Markov meta policies, does not depend on the start state $s$.  In addition, the solutions of $q$ may not be different by a constant vector.
Given an MDP and a set of options, the goal of the off-policy control problem is to find a policy that achieves $\ooptimalr$. Inter-option Differential Q-learning achieves this goal by solving the \emph{optimality} equation for SMDPs (Puterman 1994):
\begin{align}
    q(s, o) & = \sum_{s', r,\,l} \otrans(s', r,\,l \mid s, o) \big( r - \bar r \cdot l  + \max_{o'} q (s', o') \big), \label{eq: option-value optimality equation}
\end{align}
where $q$ and $\bar{r}$ denote estimates of the option-value function and the reward rate respectively. Just as in the MDP setting, $\bar r$ has $\ooptimalr$ as its unique solution, and solutions of $q$ may not be different by a constant vector.


Intra-option Differential-learning finds an optimal policy by solving the \emph{intra-option optimality} equation.
\begin{align}
    q(s, o) &= \sum_{a} \pi(a \mid s, o) \sum_{s', r} \trans (s', r \mid s, a) (r - \bar r + u^q_*(s', o)), \quad \forall\ s \in \sspace, o \in \ospace, \label{eq: intra-option-value optimality equation}
\end{align}
where
\begin{align}\label{eq: u_*}
    u^q_*(s', o) &\doteq \big( 1 - \beta(s', o) \big) q(s', o) + \beta(s', o) \max_{o'} q(s', o').
\end{align}

The following proposition shows that the set of solutions of \eqref{eq: intra-option-value optimality equation} is the same as that of \eqref{eq: option-value optimality equation}.

\begin{proposition}\label{prop: inter = intra equations}
Any solution of
\eqref{eq: option-value optimality equation} is also a solution of
\eqref{eq: intra-option-value optimality equation} and vice versa.
\end{proposition}


\section{Convergence Results}
In this section, we present convergence theories of Differential Q-learning and RVI Q-learning in weakly communicating MDPs, and theories of the two option extensions of Differential Q-learning in weakly communicating SMDPs. Empirical results verifying the convergence of the two MDP algorithms are presented in \cref{sec: empirical validation}.

Differential Q-learning updates a table of estimates $Q_t: \sspace \times \aspace \to \bbR$ as follows:
\begin{align}
    Q_{t+1}(S_t, A_t) &\doteq Q_t (S_t, A_t) + \alpha_{\nu(t, S_t, A_t)} \delta_t,  \label{eq: DQL Q update} \\
    Q_{t+1}(s, a) &\doteq Q_t (s, a),\ \forall\ s,a \neq S_t, A_t, \nonumber
\end{align}
where $\nu(t, S_t, A_t)$ is the number of times $S_t, A_t$ has been visited before time step $t$, $\{\alpha_{\nu(t, S_t, A_t)}\}$ is a step-size sequence, and $\delta_t$, the temporal-difference (TD) error, is: 
\begin{align} \label{eq: DQL delta}
    \delta_t \doteq R_{t+1} - \bar R_t + \max_{a} Q_t(S_{t+1}, a) - Q_t (S_t, A_t),
\end{align}
where $\bar R_t$ is a scalar estimate of $r_*$, updated by:
\begin{align}
    \bar R_{t+1} \doteq \bar R_t + \eta \alpha_{\nu(t, S_t, A_t)} \delta_t \label{eq: DQL bar R update},
\end{align}
and $\eta$ is a positive constant.

We now present the convergence theory of Differential Q-learning. We first state the required assumptions, which are also required by the original convergence theory of Differential Q-learning by Wan et al.~(2021a).

\begin{assumption}\label{assu: stepsize} For all $n \geq 0$, $\alpha_n > 0$, $\sum_{n = 0}^\infty \alpha_n = \infty$, and $\sum_{n = 0}^\infty \alpha_n^2 < \infty$.
\end{assumption}
\begin{assumption} \label{assu: asynchronous stepsize 1}
Let $[\cdot]$ denote the integer part of $(\cdot)$, for $x \in (0, 1)$, $\sup_n \frac{\alpha_{[xn]}}{\alpha_n} < \infty$
and $\frac{\sum_{n=0}^{[ym]} \alpha_n}{\sum_{n=0}^m \alpha_n} \to 1$
uniformly in $y \in [x, 1]$.
\end{assumption}
\begin{assumption} \label{assu: asynchronous stepsize 2}
There exists $\Delta > 0$ such that 
\begin{align*}
    \liminf_{n \to \infty} \frac{\nu(n, s, a)}{n+1} \geq \Delta,
\end{align*}
a.s., for all $s \in \sspace, a \in \aspace$.
Furthermore, for all $x > 0$, let $N(n, x) = \min \left \{m > n: \sum_{k = n+1}^m \alpha_k \geq x \right \},$
the limit $\lim_{n \to \infty} \left(\sum_{k = \nu(n, s, a)}^{\nu(N(n, x), s, a)} \alpha_k \right)/ \left (\sum_{k = \nu(n, s', a')}^{\nu(N(n, x), s', a')} \alpha_{k} \right)$
exists a.s. for all $s, s' \in \sspace, a, a' \in \aspace$.
\end{assumption}

\begin{theorem} \label{thm: Differential Q-learning}
If $\mdp$ is communicating and Assumptions~\ref{assu: stepsize}--\ref{assu: asynchronous stepsize 2} hold, the Differential Q-learning algorithm (Equations~\ref{eq: DQL Q update}--\ref{eq: DQL bar R update}) converges, almost surely, $\bar R_t$ to $\optimalr$, $Q_t$ to the set of solutions of \eqref{eq: action-value optimality equation} and
\begin{align}
    \optimalr - \bar R_0 & = \eta \left(\sum_{s, a} q(s, a) - \sum_{s, a} Q_0 (s, a) \right) \label{eq: DQL q solution determination},
\end{align}
and $r(\pi_t, s)$ to $\optimalr$, for all $s \in \sspace$, where $\pi_t$ is any greedy policy w.r.t. $Q_t$.
\end{theorem} 

\textbf{Remark:} If the MDP is weakly communicating, that is, it contains transient states, the agent eventually reaches the closed communicating state and never returns to the transient states. Elements in $Q_n$ that are associated with the closed communicating set converge to a set that depends on the values of $Q$ and $\bar R$ when the MDP reaches the closed communicating set for the first time. Other elements in $Q_n$ would only be visited for a finite number of times and can not be guaranteed to converge to their correct values by any learning algorithm. Other conclusions of the theorem remain unchanged. This observation on weakly communicating MDPs also applies to Theorems~\ref{thm: RVI Q learning}--\ref{thm: intra option}.

The update rules of RVI Q-learning are
\begin{align}
    Q_{t+1}(S_t, A_t) & \doteq Q_t(S_t, A_t) + \alpha_{\nu(t, S_t, A_t)} \delta_t(S_t, A_t) \label{eq: RVI Q update},\\
    Q_{t+1}(s, a) &\doteq Q_t (s, a),\ \forall\ s,a \neq S_t, A_t, \nonumber
\end{align}
where
\begin{align}
    \delta_t(S_t, A_t) \doteq R_{t+1} - f(Q_t) + \max_{a} Q_t(S_{t+1}, a) - Q_t(S_t, A_t), \label{eq: RVI delta}
\end{align}
and $f: \sspace \times \aspace \to \bbR$ satisfies the following assumption.
\begin{assumption}\label{assu: f}
1) $f$ is $L$-Lipschitz, 2) there exists a positive scalar $u$ s.t. $f(e) = u$ and $f(x + ce) = f(x) + cu$, and 3) $f(cx) = cf(x), \forall c \in \bbR$.
\end{assumption}

\begin{theorem} \label{thm: RVI Q learning}
If $\mdp$ is communicating and Assumptions~\ref{assu: stepsize}--\ref{assu: f} hold, then the RVI Q-learning algorithm (Equations~\ref{eq: RVI Q update}--\ref{eq: RVI delta}) converges, almost surely, $\bar R_t$ to $\optimalr$, $Q_t$ to the set of solutions of \eqref{eq: action-value optimality equation} and
\begin{align}
    \optimalr & = f(q) \label{eq: RVI q solution determination},
\end{align}
and $r(\pi_t, s)$ to $\optimalr$, for all $s \in \sspace$, where $\pi_t$ is any greedy policy w.r.t. $Q_t$.
\end{theorem} 

Now consider option extensions of Differential Q-learning. Given an SMDP $\smdp = (\sspace, \ospace, \lspace, \orspace, \otrans)$, inter-option Differential Q-learning maintains estimates of option values, and, inspired by Schweitzer (1971), updates estimates using scaled TD errors:
\begin{align}
    Q_{n+1}(\hat S_n, \hat O_n) & \doteq Q_{n}(\hat S_n, \hat O_n) + \alpha_{\nu(n, \hat S_n, \hat O_n)} \delta_{n} / L_n(\hat S_n, \hat O_n), \label{eq: Inter-option Differential TD-learning Q}\\
    Q_{n+1}(s, o) &\doteq Q_n (s, o),\ \forall\ s,o \neq \hat S_n, \hat O_n, \nonumber\\
    \bar R_{n+1} & \doteq \bar R_{n} + \eta \alpha_{\nu(n, \hat S_n, \hat O_n)} \delta_{n}/L_n(\hat S_n, \hat O_n), \label{eq: Inter-option Differential TD-learning R bar}
\end{align}
where $\nu(n, \hat S_n, \hat O_n)$ is the number of visits to state-option pair $(\hat S_n, \hat O_n)$ before stage $n$, $L_n(\cdot,\cdot)$ comes from an additional vector of estimates $L: \sspace \times \ospace \to \bbR$ that approximate the expected lengths of state-option pairs, updated by:
\begin{align}
    L_{n+1}(\hat S_n, \hat O_n) \doteq L_{n}(\hat S_n, \hat O_n) + \beta_{\nu(n, \hat S_n, \hat O_n)} (\hat L_n - L_{n}(\hat S_n, \hat O_n)), \label{eq: Inter-option Differential TD-learning L}
\end{align}
where $\{\beta_n\}$ is the other step-size sequence. 
The TD-error $\delta_n$ in \eqref{eq: Inter-option Differential TD-learning Q} and \eqref{eq: Inter-option Differential TD-learning R bar} is
\begin{align} \label{eq: Inter-option Differential Q-learning TD error}
    \delta_n & \doteq \hat R_n - L_n(\hat S_n, \hat O_n) \bar R_n + \max_o Q_n (\hat S_{n+1}, o) - Q_n(\hat S_{n}, \hat O_n),
\end{align}
\begin{theorem}\label{thm: inter option}
If $\smdp$ is communicating, Assumptions~\ref{assu: option assumption}--\ref{assu: asynchronous stepsize 2} hold, except for using $\nu(n, s, o)$ instead of $\nu(t, s, a)$, and that $0 \leq \beta_n \leq 1$, $\sum_n \beta_n = \infty$, and $\sum_n \beta_n^2 < \infty$, inter-option Differential Q-learning (Equations~\ref{eq: Inter-option Differential TD-learning Q}-\ref{eq: Inter-option Differential Q-learning TD error}) converges, almost surely,
$Q_n$ to the set of solutions of \eqref{eq: option-value optimality equation} and
\begin{align}
    \ooptimalr - \bar R_0 & = \eta \left(\sum_{s, o} q(s, o) - \sum_{s, o} Q_0(s, o) \right) \label{eq: option DQL q solution determination},
\end{align}
$\bar R_n$ to $\ooptimalr$, and $\optionr(\opi_n, s)$ to $\ooptimalr$ where $\opi_n$ is a greedy policy w.r.t. $Q_n$.
\end{theorem}


Intra-option Differential Q-learning also maintains estimates of option values. However, instead of updating the estimates using option transitions, it updates for all options using each action transition $(S_t, O_t, A_t, R_{t+1}, S_{t+1})$.
\begin{align}
    Q_{t+1}(S_t, o) &\doteq Q_{t}(S_t, o) + \alpha_{\nu(t, S_t, o)} \rho_t(o) \delta_t(o), \quad \forall\ o \in \ospace, \label{eq: Intra-option Differential TD-learning Q}\\
    Q_{t+1}(s, o) &\doteq Q_{t}(s, o), \quad \forall\ s \in \sspace, o \in \ospace, \nonumber \\
    \bar R_{t+1} &\doteq \bar R_t + \eta \sum_{o \in \ospace} \alpha_{\nu(t, S_t, o)} \rho_t(o) \delta_t(o), \label{eq: Intra-option Differential TD-learning R bar}
\end{align}
where $\{\alpha_t\}$ is a step-size sequence, $\rho_t(o) \doteq \frac{\pi(A_t | S_t, o)}{\pi(A_t | S_t, O_t)}$ is the importance sampling ratio, and:
\begin{align}
\label{eq: Intra-option Differential Q-learning TD error}
    \delta_t(o) \doteq R_{t+1} - \bar R_t + u^{Q_t}_*(S_{t+1}, o) - Q_t(S_t, o),
\end{align}
where $u_*^{Q_t}$ is defined in \eqref{eq: u_*}.

\begin{theorem}\label{thm: intra option}
If $\smdp$ is communicating, Assumptions~\ref{assu: option assumption}--\ref{assu: asynchronous stepsize 2} hold, except for using $\nu(t, s, o)$ instead of $\nu(t, s, a)$, intra-option Differential Q-learning (Equations~\ref{eq: Intra-option Differential TD-learning Q}-\ref{eq: Intra-option Differential Q-learning TD error}) converges, almost surely,
$Q_t$ to the set of solutions of \eqref{eq: option-value optimality equation} and \eqref{eq: option DQL q solution determination},
$\bar R_t$ to $\ooptimalr$, and $\optionr(\opi_t, s)$ to $\ooptimalr$ where $\opi_t$ is a greedy policy w.r.t. $Q_t$.
\end{theorem}

\section{Characterization of the Solution Set}

In this section, we characterize the sets that the algorithms described in the previous section converge to. This section plays a key role in showing their convergence.

We consider the set of solutions of $q$ in the SMDP optimality equation (\eqref{eq: option-value optimality equation}) and 
\begin{align}
    \ooptimalr & = f(q) \label{eq: q solution determination},
\end{align}
where $f: \sspace \times \ospace \to \bbR$ satisfies \cref{assu: f}.
It is clear that \eqref{eq: option-value optimality equation} generalizes over \eqref{eq: action-value optimality equation} and \eqref{eq: q solution determination} generalizes over \eqref{eq: DQL q solution determination}, \eqref{eq: RVI q solution determination}, and \eqref{eq: option DQL q solution determination}. And thus the characterization of $\GDiffQsolutionq$ applies to the sets that action/option values in the aforementioned algorithms are claimed to converge to in Theorems\ \ref{thm: Differential Q-learning}--\ref{thm: intra option}.

It is known that if the SMDP is weakly communicating, \eqref{eq: option-value optimality equation} has $\ooptimalr$ as its unique solution of $\bar r$. For $q$, it has been shown by Schweitzer \& Federgruen (1978) (we will refer to this work multiple times and thus we use a shorthand ``S\&F'' for simplicity from now on) in their Theorem 4.2 that the set of solutions of $q$ in \eqref{eq: option-value optimality equation} is closed, unbounded, connected, and possibly non-convex. The next theorem characterizes $\GDiffQsolutionq$.
\begin{theorem} \label{thm: characterize Q}
If the SMDP is weakly communicating and \cref{assu: f} holds, $\GDiffQsolutionq$ is non-empty, closed, bounded, connected, and possibly non-convex.
\end{theorem}
Before presenting the proof, first note that our convergence proof does not rely on the convexity property and we defer the proof of non-convexity to \cref{sec: non-convexity}.
\begin{proof}

First, $\GDiffQsolutionq$ is non-empty. To see this point, note that for any solution of $q$ in \eqref{eq: option-value optimality equation}, $q_*$, $q_* + ce$ is also a solution for any $c \in \bbR$ and thus there must be a $c$ such that \eqref{eq: q solution determination} holds because $f(x + ce) = f(x) + cu$ for any $x, c$.  

$\GDiffQsolutionq$ is closed because the set of solutions of $q$ in \eqref{eq: option-value optimality equation} is closed by S\&F, the set of solutions of $q$ in \eqref{eq: q solution determination} is closed because $f$ is Lipschitz  and is thus continuous, and the intersection of two closed sets is closed.

\textbf{Boundedness}

We now show that $\GDiffQsolutionq$ is bounded. For any $q \in \GDiffQsolutionq$, let $v(s) \doteq \max_o q(s, o)$. Rewrite the option-value optimality equation (\cref{eq: option-value optimality equation}) using $v$ instead of $q$, we have,
\begin{align}
    v(s) &= \max_o \sum_{s', r, l} \otrans(s', r, l \mid s, o) (r - l \ooptimalr + v(s')), \forall\ s \in \sspace \label{eq: state value optimality equation}.
\end{align}
The above equation is known as the state-value optimality equation.

It is easy to verify that
\begin{align} \label{eq: linear transformation q and v}
    q(s, o) = \sum_{s', r, l} \otrans(s', r, l \mid s, o) (r - l \ooptimalr + v(s')).
\end{align}
Using this fact, rewrite \cref{eq: q solution determination} using $v$ instead of $q$, we have,
\begin{align}
    \ooptimalr & = f(\tilde r + \tilde Pv) \label{eq: v solution determination},
\end{align}
where
\begin{align*}
    \tilde r(s, o) & \doteq \sum_{s', r, l} \otrans(s', r, l \mid s, o) (r - l \ooptimalr ),\\
    \tilde P(s, o, s') & \doteq \sum_{r, l} \otrans(s', r, l \mid s, o).
\end{align*}
Denote the set of solutions of $v$ in \eqref{eq: state value optimality equation} by $\calV$. Denote the set of solutions of $v$ in \cref{eq: state value optimality equation} and \cref{eq: v solution determination} by $\calV_\infty$. If $\calV_\infty$ is bounded, $\GDiffQsolutionq$ is also bounded because any $q \in \GDiffQsolutionq$ can be obtained from a solution of $v \in \calV_\infty$ with a linear operation in view of \cref{eq: linear transformation q and v}.

In order to show boundedness, We will need the following two lemmas to proceed. These two lemmas are similar to Theorem 4.1 (c) and Theorem 5.1 in S\&F, except that 1) the '$\max$' operates over the set of all optimal policies, instead of the set of all deterministic optimal policies as in Theorem 4.1 (c), and 2) we consider the set of weakly communicating SMDPs while S\&F considers general multi-chain SMDPs. The proofs are also essentially the same. For completeness, we provide the proofs for these two lemmas in Sections~\ref{proof of lemma: s&f 1978 Theorem 4.1 c}, \ref{proof of lemma: s&f 1978 5.1 d}. 

To formally state the \cref{lemma: s&f 1978 Theorem 4.1 c}, we will first introduce some definitions.

For any $\opi \in \osetpolicies$, let $\otransmatrix$ denote the $\cardS \times \cardS$ transition probability matrix under policy $\opi$. That is,
\begin{align} \label{eq: P_pi definition}
    \otransmatrix(s, s') \doteq \sum_{o, r, l} \opi(o \mid s) \otrans(s', r, l \mid s, o).
\end{align}

Let $\olimitingmatrix$ be the \emph{limiting matrix} of $\otransmatrix$, which is the Cesaro limit of the sequence $\{\otransmatrix^i\}_{i = 1}^\infty$:
\begin{align}\label{eq: limiting matrix}
    \olimitingmatrix \doteq \lim_{n \to \infty} \frac{1}{n} \sum_{i = 0}^{n-1} \otransmatrix^i.
\end{align}
Because $\sspace$ is finite, the Cesaro limit exists and $\olimitingmatrix$ is a stochastic matrix (has row sums equal to 1). 

Let $\oonestepl(s) \doteq \sum_{o, s', r, l} \opi(o \mid s) \otrans(s', r, l \mid s, o) l$. And let the fundamental matrix $\ofundmatrix \doteq (I - \olimitingmatrix + \olimitingmatrix)^{-1} = I + \lim_{\gamma \uparrow 1} \sum_{n=1}^\infty \gamma^n (\otransmatrix^n - \olimitingmatrix)$. 

\begin{lemma}\label{lemma: s&f 1978 Theorem 4.1 c}
If the SMDP is weakly communicating, $v$ is a solution of \eqref{eq: state value optimality equation} if and only if 
\begin{align}
    v(s) = \max_{\opi \in \osetoptimalpolicies} [\ofundmatrix (\oonestepr - \oonestepl \ooptimalr) + \olimitingmatrix v](s), \forall\ s \in \sspace \label{eq: lemma 1 eq}.
\end{align}
\end{lemma}

In order to state \cref{lemma: s&f 1978 5.1 d} formally, we need the following definitions. Define the Bellman error for a state $s \in \sspace$ given a policy $\opi \in \osetpolicies$ and some $v \in \R^{\cardS}$, $b_{v, \opi}(s)$, as follows:
\begin{align*}
    b_{v, \opi}(s) & \doteq [\oonestepr - \oonestepl \ooptimalr + \olimitingmatrix v - v] (s).
\end{align*}

Define $\orecurrentstates$ as the set of recurrent states for $\olimitingmatrix$. That is,
\begin{align*}
    \orecurrentstates & \doteq \{s \mid \olimitingmatrix(s, s) > 0\}.
\end{align*}
Let $\osetoptimalpolicies$ denote the set of optimal meta policies. Define $\ooptimalrecurrentstates$ as the set of states that are recurrent under some optimal meta policy:
\begin{align*}
    \ooptimalrecurrentstates & \doteq \{s \mid s \in \orecurrentstates \text{ for some policy } \opi \in \osetoptimalpolicies\} .
\end{align*}
By Theorem 3.2 (b) in S\&F, there exists a policy $\opistar \in \osetoptimalpolicies$ such that $R_{\opistar} = \ooptimalrecurrentstates$. For any $\opi \in \osetpolicies$, let $n(\opi)$ be the number of recurrent classes for $\olimitingmatrix$. Further, define the least number of recurrent classes induced by any optimal meta policy that induces the set of recurrent states $\ooptimalrecurrentstates$:
\begin{align*}
    \ooptimalnbrecurrentclasses \doteq \min \{n(\opi) \mid \opi \in \osetoptimalpolicies, \orecurrentstates = \ooptimalrecurrentstates\}.
\end{align*}
Denote $\osetpolicies_{**}$ as the set of policies that have $\ooptimalrecurrentstates$ as their sets of recurrent states and that have $\ooptimalnbrecurrentclasses$ recurrent classes. That is,
\begin{align*}
    \osetpolicies_{**} \doteq \{\opistar \in \osetoptimalpolicies \mid \orecurrentstates = \ooptimalrecurrentstates, n(\opistar) = \ooptimalnbrecurrentclasses\}.
\end{align*}
Theorem 3.2 (d) by S\&F shows that all policies within $\osetpolicies_{**}$ share the same collection of recurrent classes.
Denote the collection of recurrent classes as $\{R_{*\alpha} \mid 1, 2, \cdots, \ooptimalnbrecurrentclasses\}$. The following lemma shows that the solution set of \eqref{eq: state value optimality equation} has $\ooptimalnbrecurrentclasses$ degrees of freedom.

\begin{lemma}\label{lemma: s&f 1978 5.1 d}
If the SMDP is weakly communicating, suppose $v$ and $v + x$ are both solutions of \eqref{eq: state value optimality equation}, then there exists $\ooptimalnbrecurrentclasses$ constants $y_1, y_2, \cdots, y_{\ooptimalnbrecurrentclasses}$ such that
\begin{align}
    x(s)  & =
    y_\alpha, i \in R_{*\alpha}, \quad \alpha = 1, \dots, \ooptimalnbrecurrentclasses \label{eq: s&f 1978 5.1 1} \\
    x(s) & = \max_{\opi \in \osetoptimalpolicies} [\ofundmatrix b_{v, \opi}](s) + \sum_{\beta = 1}^{\ooptimalnbrecurrentclasses} \left(\sum_{s' \in R_{*\beta}} \olimitingmatrix(s, s')\right) y_\beta, \quad s \in \sspace \backslash \ooptimalrecurrentstates, \label{eq: s&f 1978 5.1 2}\\
    y_\alpha & \geq [\ofundmatrix b_{v, \opi}](s) + \sum_{\beta = 1}^{\ooptimalnbrecurrentclasses} \left(\sum_{s' \in R_{*\beta}} \olimitingmatrix(s, s') \right) y_\beta, \quad \alpha = 1, \dots, \ooptimalnbrecurrentclasses, s \in R_{*\alpha}, \opi \in \osetoptimalpolicies . \label{eq: s&f 1978 5.1 3}
\end{align}
\end{lemma}

For any $\beta \in \{1, 2, \cdots, \ooptimalnbrecurrentclasses\}$, note that there exists a policy $\opi(\beta)$ such that $R_{*\beta}$ is the only one recurrent class under $\opi(\beta)$. To see this point, note that the SMDP is weakly communicating, and thus we can modify $\opistar$ to obtain a new meta policy such that all states except for those in $R_{*\beta}$ are transient. 

Based on the above observation and \cref{lemma: s&f 1978 5.1 d}, for any $v \in \calV$, we have for any given $\beta \in \{1, \dots, \ooptimalnbrecurrentclasses\}$, there exists a $\opi(\beta)$, such that
\begin{align*}
    y_\alpha &\geq \max_{s \in R_{*\alpha}, }Z_{\opi(\beta)} b_{v, \opi(\beta)}(s) + \sum_{s' \in R_{*\beta}} P^\infty_{\opi(\beta)}(s, s')  y_\beta \\
    & = \max_{s \in R_{*\alpha}, }Z_{\opi(\beta)} b_{v, \opi(\beta)}(s) +  y_\beta, \quad \forall\ \alpha = 1, \cdots, \ooptimalnbrecurrentclasses.
\end{align*}
The first term $\max_{s \in R_{*\alpha}, }Z_{\opi(\beta)} b_{v, \opi(\beta)}(s)$ is a constant given $v$ and $\opi(\beta)$.
Therefore we see that, for any other solution $v + x$ of \eqref{eq: state value optimality equation}, if $y_\beta$ is arbitrarily large then $y_\alpha, \forall\ \alpha = 1, \cdots, \ooptimalnbrecurrentclasses$ should also be arbitrarily large. This would violate the Lipschitz assumption on $f$. To see this point, let
\begin{align}
    \tilde f(v) \doteq f(\tilde r + \tilde P v). \label{eq: tilde f}
\end{align}
Let $L$ be a Lipschitz constant of $f$. $L$ is also a Lipschitz constant of $\tilde f$ because $\tilde P$ is a stochastic matrix and is thus a non-expansion. Choose a $v_1 \in \calV_\infty$ and a $\tilde v_2 \in \calV_\infty$. a $\tilde v_2 \in \calV_\infty$, denote $m = \norm{v_1 - \tilde v_2}$. Choose a sufficiently large $c > 0$ such that $cu > L\norm{v_1+ ce - v_2} = L\norm{v_1+ ce - \tilde v_2 - ce} = Lm$, where $v_2 \doteq \tilde v_2 + ce$. Given this choice of $c$, using $\tilde f(v_1) = \tilde f(v_2) = \ooptimalr$ and $\tilde f(v_1 + ce) = \tilde f(v_1) + cu$, we have $\norm{\tilde f(v_1 + ce) - \tilde f(v_2)} = \norm{\tilde f(v_1) + c u - \tilde f(v_2)} = c u > L\norm{v_1+ce - v_2}$. This inequality suggests that $\tilde f$ is not Lipschitz continuous with a Lipschitz constant $L$ and thus violates our assumption. Because the choice of $\beta$ is arbitrary, $\calV_\infty$ is upper bounded.

In addition, because the choice of $\beta$ is arbitrary, we have for any $\alpha \in \{1, \dots, \ooptimalnbrecurrentclasses\}$,
\begin{align*}
    y_\alpha &\geq \max_{\beta \in \{1, \dots, \ooptimalnbrecurrentclasses\}} \max_{s \in R_{*\alpha}, }Z_{\opi(\beta)} b_{v, \opi(\beta)}(s) +  y_\beta
\end{align*}
If $y_\alpha$ is chosen to be arbitrarily small then $y_\beta$ should also be arbitrarily small for all $\beta = 1, \cdots, \ooptimalnbrecurrentclasses$ but again this is not allowed due to \eqref{eq: v solution determination} for the same reason as mentioned in the previous paragraph. Therefore $y_\alpha, \forall\ \alpha \in \{1, \dots, \ooptimalnbrecurrentclasses\}$ can not be arbitrarily small. Thus $\calV_\infty$ is lower bounded. Combining the upper bound and lower bound, $\calV_\infty$ is bounded. Therefore $\GDiffQsolutionq$ is also bounded.

\textbf{Connectedness}

We now show that $\GDiffQsolutionq$ is connected. To this end, again it is enough to show that $\calV_\infty$ is connected. 

Define a function that takes a $v \in \calV$ as input and produces an element in $\calV_\infty$ as output. Specifically, let $g: \calV \to \calV_\infty$ with $g(v) = v + x e$, where $x$ is the solution of $\tilde f(v + xe) = \ooptimalr$ and $\tilde f$ is defined in \eqref{eq: tilde f}. Note that $x$ is unique given $v$ because $\ooptimalr = \tilde f(v + xe) = \tilde f(v) + xu$ and thus $x = (\ooptimalr - \tilde f(v)) / u$.

We now show that $g$ is Lipschitz continuous. Consider any $v_1, v_2 \in \calV$. Let $x_1, x_2$ satisfy $v_1 + x_1e = g(v_1)$ and $v_2 + x_2e = g(v_2)$ respectively. Again $x_1, x_2$ are unique given $v_1, v_2$. Note that
\begin{align*}
    & \tilde f(v_1 + x_1e) - \tilde f(v_2 + x_1e) \\
    & = \tilde f(v_1 + x_1e) - \tilde f(v_2 + x_2e) + (x_1 - x_2)u \\
    & = \ooptimalr - \ooptimalr + (x_1 - x_2) u \\
    & = (x_1 - x_2)u
\end{align*}
\begin{align*}
    |x_1 - x_2| & = |\tilde f(v_1 + x_1e) - \tilde f(v_2 + x_1e)| / u \\
    & \leq L \norm{v_1 + x_1e - v_2 - x_1e} / u\\
    & = L \norm{v_1 - v_2} /u 
\end{align*}
\begin{align*}
    \norm{g(v_1) - g(v_2)} = \norm{v_1 + x_1e - v_2 - x_2e} \leq \norm{v_1 - v_2} + \norm{x_1e - x_2e} = (1 + L /u) \norm{v_1 - v_2}.
\end{align*}
Therefore $g$ is Lipschitz continuous with Lipschitz constant $1 + L /u$.

Finally, because $\calV$ is connected and the image of any continuous function on a connected set is connected, $g(\calV)$ is connected. Note that every point in $g(\calV)$ belongs to $\calV_\infty$ by definition. Every point in $\calV_\infty$ also belongs to $g(\calV)$. To see this point, pick any $v \in \calV_\infty$, we can see that $v \in \calV$ and that $g(v) = v$ (note that $x = 0$ given that $v \in \calV_\infty$). Thus $v \in g(\calV)$. Therefore $\calV_\infty = g(\calV)$ is connected.

Given that $\calV_\infty$ is connected, $\GDiffQsolutionq$ should also be connected because $\GDiffQsolutionq$ is a linear transformation of $\calV_\infty$ (see Equation \ref{eq: linear transformation q and v}).

\end{proof}

The other result we will need to use to show the convergence of the four algorithms introduced in the previous section is the following one. With this result, the stability of the algorithms can be established using the result by Borkar and Meyn (2000) (see also, Section 3.2 by Borkar 2009).

\begin{lemma} \label{lemma: 0 reward MDP has unique solution}
If an SMDP is weakly communicating and all rewards are $0$, $0$ is the only element in $\GDiffQsolutionq$.
\end{lemma}
\begin{proof}
Given a weakly communicating SMDP, by \cref{lemma: s&f 1978 Theorem 4.1 c}, any solution of the state-value optimality equation (\eqref{eq: state value optimality equation})
satisfies
\begin{align*}
    v(s) \geq \sum_{s' \in \ooptimalrecurrentstates} \olimitingmatrixstar(s, s') v(s'), \forall\ s \in \ooptimalrecurrentstates, \opistar \in \osetoptimalpolicies,
\end{align*}
where $\ooptimalrecurrentstates$ is defined right after \cref{lemma: s&f 1978 Theorem 4.1 c}. Also, because all rewards are $0$, $\ooptimalrecurrentstates$ is the closed communicating class and $\osetoptimalpolicies = \osetpolicies$ contains all stationary policies.

Pick an arbitrary $v_* \in \calV$ and an arbitrary policy $\opi \in \osetpolicies$. For each recurrent class $C$ under $\otransmatrix$, we have, by \cref{lemma: s&f 1978 Theorem 4.1 c}, $\forall\ s \in C$, $v_*(s) \geq \sum_{s' \in C} \ostationarydist^C(s') v_*(s')$, where $\ostationarydist^C$ denotes the stationary distribution of $\opi$ in the recurrent class $C$. The r.h.s. only involves class $C$ because starting from a state $s \in C$ the MDP can not leave $C$. Because the choice of $s$ is arbitrary, $\min_{s \in C} v_*(s) \geq \sum_{s' \in C} \ostationarydist^C(s') v_*(s') \geq \min_{s \in C} v_*(s).$
Thus $\sum_{s' \in C} \ostationarydist^C(s') v_*(s') = \min_{s \in C} v_*(s)$. In addition, because $\ostationarydist^C(s) > 0$ for all $s \in C$, $v_*(s') = v(s), \forall s, s' \in C$.


Now for any $s, s' \in \ooptimalrecurrentstates$, there must exist a $\opi \in \osetoptimalpolicies = \osetpolicies$ such that there is a path from $s$ to $s'$ and a path from $s'$ to $s$, because $s, s'$ are in the same communicating class. Therefore $s, s'$ are in the same recurrent class under $\otransmatrix$. Thus we conclude that $v_*(s) = v_*(s')$. Therefore $\forall\ s, s' \in \ooptimalrecurrentstates$, $v_*(s) = v_*(s')$. The transient states values are uniquely determined by values of states in $\ooptimalrecurrentstates$. And in this case they are all equal to the values of states in the communicating class because all rewards are zero. Thus the solution set of $v$ in the state-value optimality equation (\eqref{eq: state value optimality equation}) is $\{ce : \forall c \in \bbR\}$.

Now consider the solution set of the option-value optimality equation (\eqref{eq: option-value optimality equation}). Let $v(s) = \max_o q(s, o)$, then \eqref{eq: option-value optimality equation} transforms to the state-value optimality equation. Therefore for any two solutions of $q$ in \eqref{eq: option-value optimality equation}, $q_1$ and $q_2$, $\max_o q_1(\cdot, o) = \max_o q_2(\cdot, o) + ce$ for some $c$. Furthermore, let $q_*$ be any solution of $q$, $q_*(s, o_1) = q_*(s, o_2), \forall\ s \in \sspace, o_1, o_2 \in \ospace$ because $\forall\ s \in \sspace, o \in \ospace$:
\begin{align*}
    q_*(s, o) & = \sum_{s', r} \trans(s', r \mid s, o) \max_{o'} q_*(s', o')\\
    & = \sum_{s', r} \trans(s', r \mid s, o) \max_{o'} q_*(s, o')\\
    & = \max_{o'} q_*(s, o').
\end{align*}
Thus the solution set of $q$ in the option-value optimality equation (\eqref{eq: option-value optimality equation}) is also $\{ce : \forall c \in \bbR\}$.
Given \eqref{eq: q solution determination}, $cu = c f(e) = f(ce) = \ooptimalr = 0$ and $u > 0$ implies that $c = 0$. Therefore $0$ is the unique solution of $q$. The lemma is proved.
\end{proof}

\section{Proof sketch of \cref{thm: Differential Q-learning}-\cref{thm: intra option}}\label{sec:convergence-proof-diffq}
In this section, we sketch the proof of Theorems\ \ref{thm: Differential Q-learning}-\ref{thm: intra option}.  It has been shown that all four algorithms introduced above are special cases of the \emph{General RVI Q algorithm} (Wan et al. 2021a,b). They also showed that General RVI Q converges under an assumption that is not satisfied for weakly communicating MDPs/SMDPs. In order to show convergence for weakly communicating MDPs/SMDPs, we replace this assumption with three weaker assumptions that are satisfied for these MDPs/SMDPs. All other assumptions are the same as those used by Wan et al. (2021a,b) and can be verified for all four algorithms using their arguments. We present General RVI Q and prove its convergence with the three new assumptions in \cref{sec: General RVI Q}. The next step of the proof would be verifying the three new assumptions when casting General RVI Q to each of the four algorithms. This should be straightforward given our \cref{thm: characterize Q} and \cref{lemma: 0 reward MDP has unique solution}. We defer this part to \cref{sec: Verifying assumptions}. Given that the three assumptions are verified, we have the conclusion part of the convergence theorem of General RVI Q holds for each of the four algorithms. The convergence of the reward rates of greedy policies w.r.t. the action/option-values follows the convergence of these values and is shown in \cref{sec: convergence of r(pi_t)}.

\section{Conclusions}
In this paper, we provide, for the first time, convergence results of off-policy average-reward control algorithms in weakly communicating MDPs, which are known to be the most general class of MDPs in which it is possible that a learning algorithm can guarantee to obtain an optimal policy. Specifically, we show two existing algorithms, RVI Q-learning and Differential Q-learning, converge in weakly communicating MDPs. As an extension, we also showed two off-policy average-reward options learning algorithms converge if the SMDP induced by the options is weakly communicating. 


\section*{Acknowledgements}
The authors were generously supported by DeepMind, Amii, NSERC, and CIFAR. The authors wish to thank Huizhen Yu and Abhishek Naik for discussing several important related papers and discussing ideas to address the technical challenges in the proof. Computing resources were provided by Compute Canada.

\newpage
\section*{References}

\parskip=5pt
\parindent=0pt
\def\hangin{\hangindent=0.2in}

\hangin
Abounadi, J., Bertsekas, D., Borkar, V. S. (2001).  Learning Algorithms for Markov Decision Processes with Average Cost. \emph{SIAM Journal on Control and Optimization}.

\hangin
Bertsekas, D. P. (2007). \emph{Dynamic Programming and Optimal Control third edition, volume II}. Athena Scientific.

\hangin
Bertsekas, D. P., Tsitsiklis, J. N. (1996). \emph{Neuro-dynamic Programming}. Athena Scientific.

\hangin
Borkar, V. S. (1998). Asynchronous Stochastic Approximations. \emph{SIAM Journal on Control and Optimization}. 

\hangin
Borkar, V. S. (2009). \emph{Stochastic Approximation: A Dynamical Systems Viewpoint}. Springer.

\hangin
Das, T. K., Gosavi, A., Mahadevan, S. Marchalleck, N. (1999). Solving semi-Markov decision problems using average reward reinforcement learning. \emph{Management Science}.

\hangin
Gosavi, A. (2004). Reinforcement learning for long-run average cost. \emph{European Journal of Operational Research}.

\hangin
Puterman, M. L. (1994). \emph{Markov Decision Processes: Discrete Stochastic Dynamic Programming.} John Wiley \& Sons.

\hangin
Ren, Z., Krogh, B. H. (2001). Adaptive control of Markov chains with average cost. \emph{IEEE Transactions on Automatic Control}.

\hangin
Schwartz, A. (1993). A reinforcement learning method for maximizing undiscounted rewards. In \emph{Proceedings of the International Conference on Machine Learning}.

\hangin
Schweitzer, P. J. (1971). Iterative solution of the functional equations of undiscounted Markov renewal programming. \emph{Journal of Mathematical Analysis and Applications}.

\hangin
Schweitzer, P. J., \& Federgruen, A. (1978). The Functional Equations of Undiscounted Markov Renewal Programming. \emph{Mathematics of Operations Research}.

\hangin
Singh, S. P. (1994). Reinforcement learning algorithms for average-payoff Markovian decision processes. In \textit{Proceedings of the AAAI Conference on Artificial Intelligence}.

\hangin
Sutton, R. S., Precup, D., Singh, S. (1999). Between MDPs and Semi-MDPs: A Framework for Temporal Abstraction in Reinforcement Learning. \emph{Artificial Intelligence}.

\hangin
Sutton, R. S., Barto, A. G. (2018). \emph{Reinforcement Learning: An Introduction.} MIT Press.

\hangin
Wan, Y., Naik, A., Sutton, R. S. (2021a). Learning and Planning in Average-Reward Markov Decision Processes.  \textit{International Conference on Machine Learning}.

\hangin
Wan, Y., Naik, A., Sutton, R. S. (2021b). Average-Reward Learning and Planning with Options.  \textit{Conference on Neural Information Processing Systems}.


\newpage
\appendix
\section{Proofs}
\subsection{Proof of \cref{prop: finite expectation and variance}}
Note that the execution time of each option $o \in \ospace$ is the return of executing this option's policy in a stochastic shortest path MDP (SSP-MDP, Bertsekas 2007) with state space $\sspace + \{\perp\}$ ($\perp$ is the ``terminal'' state of the SSP-MDP), action space $\aspace$ and transition function $\otrans$ satisfying:
\begin{align*}
    \otrans(s', 1 \mid s, a) & \doteq (1 - \beta(s', o)) \sum_{r} \trans(s', r \mid s, a), \quad \forall\ s, s' \neq \perp, a \in \aspace, \\
    \otrans(\perp, 1 \mid s, a) & \doteq \sum_{s'} \beta(s', o) \sum_{r} \trans(s', r \mid s, a), \quad \forall\ s \neq \perp, a \in \aspace,\\
    \otrans(\perp, 0 \mid \perp, a) & \doteq 1, \quad \forall\ a \in \aspace.
\end{align*}
Also, note that by \cref{assu: option assumption}, the option's policy is a 'proper' policy (Bertsekas 2007) in the SSP-MDP. That is, when using the policy, the MDP reaches the terminal state eventually regardless of the start state.
Because the expected value of every proper policy of an SSP-MDP exists and is finite (Section 2.1 of Bertsekas (2007)), the expected value of the execution time of option $o$ exists. The existence of the variance can be shown using similar arguments as those used to show the existence of the expectation in Section 2.1 of Bertsekas (2007).

Similarly, the cumulative reward of each option $o$ is the return of executing $o$'s policy in an SSP-MDP with state space $\sspace + \{\perp\}$, action space $\aspace$, and transition function $\otrans$ satisfying:
\begin{align*}
    \otrans(s', r \mid s, a) & \doteq (1 - \beta(s', o)) \trans(s', r \mid s, a), \quad \forall\ s, s' \neq \perp, a \in \aspace \\
    \otrans(\perp, r \mid s, a) & \doteq \sum_{s'} \beta(s', o) \trans(s', r \mid s, a), \quad \forall\ s \neq \perp, a \in \aspace\\
    \otrans(\perp, 0 \mid \perp, a) & \doteq 1, \quad \forall\ a \in \aspace.
\end{align*}
Again the option's policy is proper and the expected value of the cumulative reward of option $o$ exists. So does the variance of the cumulative reward.

\subsection{Proof of \cref{prop: inter = intra equations}}
By the definition of $\otrans$ of the SMDP induced by choosing options in an MDP,
\begin{align*}
    & \otrans(\tilde s, \tilde r, \tilde l \mid s, o) = \sum_{a} \pi(a \mid s, o) \sum_{s', r} \trans(s', r \mid s, a) \\
    & [\beta(s', o) \bfI(\tilde s = s', \tilde r = r, \tilde l = 1) + (1 - \beta(s', o)) \otrans(\tilde s, \tilde r - r, \tilde l - 1 \mid s', o)]
\end{align*}

\begin{align*}
    & q(s, o) \\
    &= \sum_{\tilde  s, \tilde  r, \tilde l} \otrans (\tilde s, \tilde r, \tilde l \mid s, o) (\tilde r - \tilde l \bar r + \max_{o'} q(s', o'))\\
    & = \sum_{\tilde s, \tilde r, \tilde l} \sum_{a} \pi(a \mid s, o) \sum_{s', r} \trans (s', r \mid s, a) \\
    & [\beta(s', o) \bfI(\tilde  s = s', \tilde r = r, \tilde l = 1) + (1 - \beta(s', o)) \otrans (\tilde s, \tilde r - r, \tilde l - 1 \mid s', o)] (\tilde r - \tilde l \bar r + \max_{o'} q(s', o'))\\
    & = \sum_{a} \pi(a \mid s, o) \sum_{s', r} \trans (s', r \mid s, a) \beta(s', o) (r - \bar r + \max_{o'} q(s', o')) \\
    & + \sum_{a} \pi(a \mid s, o) \sum_{s', r} \trans (s', r \mid s, a) (1 - \beta(s', o)) \sum_{\tilde s, \tilde r, \tilde l} \otrans (\tilde s, \tilde r - r, \tilde l - 1 \mid s', o)] (\tilde r - \tilde l \bar r + \max_{o'} q(s', o'))\\
    & = \sum_{a} \pi(a \mid s, o) \sum_{s', r} \trans (s', r \mid s, a) \beta(s', o) (r - \bar r + \max_{o'} q(s', o')) \\
    & + \sum_{a} \pi(a \mid s, o) \sum_{s', r} \trans(s', r \mid s, a) (1 - \beta(s', o)) \sum_{\tilde s, \tilde r, \tilde l} \otrans (\tilde s, \tilde r, \tilde l \mid s', o)] (\tilde r + r - (\tilde l + 1) \bar r + \max_{o'} q(s', o'))\\
    & = \sum_{a} \pi(a \mid s, o) \sum_{s', r} \trans(s', r \mid s, a) \beta(s', o) (r - \bar r + \max_{o'} q(s', o')) \\
    & + \sum_{a} \pi(a \mid s, o) \sum_{s', r} \trans(s', r \mid s, a) (1 - \beta(s', o)) \left(r - \bar r +  \sum_{\tilde s, \tilde r, \tilde l} \otrans (\tilde s, \tilde r, \tilde l \mid s', o) (\tilde r - \tilde l \bar r + \max_{o'} q(s', o')) \right)\\
    & = \sum_{a} \pi(a \mid s, o) \sum_{s', r} \trans (s', r \mid s, a) \beta(s', o) (r - \bar r + \max_{o'} q(s', o')) \\
    & + \sum_{a} \pi(a \mid s, o) \sum_{s', r} \trans (s', r \mid s, a) (1 - \beta(s', o)) \left(r - \bar r +  q(s', o) \right)\\
    & = \sum_{a} \pi(a \mid s, o) \sum_{s', r} \trans (s', r \mid s, a) (r - \bar r + u^q_*(s', o))
\end{align*}
Therefore any solution of  \eqref{eq: option-value optimality equation} must be a solution of \eqref{eq: intra-option-value optimality equation}-\eqref{eq: u_*} and vice versa.

\subsection{Proof of \cref{lemma: s&f 1978 Theorem 4.1 c}}\label{proof of lemma: s&f 1978 Theorem 4.1 c}

\textbf{Part 1: \eqref{eq: state value optimality equation} $\implies$ \eqref{eq: lemma 1 eq}:}

Choose any solution of \eqref{eq: state value optimality equation}, $v_*$, and choose any $\opi \in \osetoptimalpolicies$. 
Using Theorem 3.1 (e) by S\&F,
\begin{align*}
    0 & \geq \oonestepr - \oonestepl \ooptimalr + \otransmatrix v_* - v_*\\
    (I - \otransmatrix) v_* & \geq \oonestepr - \oonestepl \ooptimalr
\end{align*}
(in weakly communicating SMDPs, the ``$L$'' set in Theorem 3.1 (e) is $\ospace$). 

Using the above inequality and Lemma 2.1 by S\&F, we have 
\begin{align*}
    (I - \olimitingmatrix)v_* & \geq \ofundmatrix (\oonestepr - \oonestepl \ooptimalr) \\
    v_* & \geq \ofundmatrix (\oonestepr - \oonestepl \ooptimalr) + \olimitingmatrix v_*
\end{align*}

Because $\opi$ can be any element in $\osetoptimalpolicies$,
\begin{align}
    v_*(s) \geq \sup_{\opi \in \osetoptimalpolicies} [\ofundmatrix (\oonestepr - \oonestepl \ooptimalr) + \olimitingmatrix v_*](s), \forall\ s \in \sspace. \label{eq: lemma 1 eq 1}
\end{align}

By Theorem 4.1 (c) in S\&F, there exists a deterministic optimal policy $\opi$, which is apparently an element of $\osetoptimalpolicies$ such that 
\begin{align*}
    v_* = \ofundmatrix (\oonestepr - \oonestepl \ooptimalr) + \olimitingmatrix v_*.
\end{align*}
This result, along with \eqref{eq: lemma 1 eq 1} shows that \begin{align*}
    v_*(s) = \max_{\opi \in \osetoptimalpolicies} [\ofundmatrix (\oonestepr - \oonestepl \ooptimalr) + \olimitingmatrix v_*](s), \forall\ s \in \sspace.
\end{align*} 


Part 1 is proven.

\textbf{Part 2: \eqref{eq: lemma 1 eq} $\implies$ \eqref{eq: state value optimality equation}:}

Conversely, if $v$ satisfies \eqref{eq: lemma 1 eq}, define
\begin{align}
    \tilde v(s) \doteq \max_{o} \sum_{s', r, l} \otrans(s', r, l \mid s, o) (r - l \ooptimalr + v(s')), \forall\ s \in \sspace. \label{eq: lemma 1 eq 2}
\end{align}
We first show that $\tilde v \geq v$. For any $\opi \in \osetoptimalpolicies$,
\begin{align*}
    \tilde v(s) & = \max_{o} \sum_{s', r, l} \otrans(s', r, l \mid s, o) (r - l \ooptimalr + v(s'))\\
    & \geq [\oonestepr - \oonestepl \ooptimalr + \otransmatrix v](s)  && \text{because of the ``$\max$''}\\
    & \geq [\oonestepr - \oonestepl \ooptimalr + \otransmatrix (\ofundmatrix (\oonestepr - \oonestepl \ooptimalr) + \olimitingmatrix v)](s) && \text{by \eqref{eq: lemma 1 eq}}\\
    & = [(I + \otransmatrix \ofundmatrix) (\oonestepr - \oonestepl \ooptimalr) + \olimitingmatrix v](s) && \text{rearranging terms}\\
    & = [(\ofundmatrix + \olimitingmatrix) (\oonestepr - \oonestepl \ooptimalr) + \olimitingmatrix v](s) && \text{by Equation 2.2 in S\&F}\\
    & = [\ofundmatrix(\oonestepr - \oonestepl \ooptimalr) + \olimitingmatrix v](s). && \text{by Theorem 3.1 (a) in S\&F}\\
    & = v(s) && \text{by \eqref{eq: lemma 1 eq}}
\end{align*}
We now show that $\tilde v \leq v$, which, together with $\tilde v \geq v$, implies $\tilde v = v$.

Let $\opi$ be a deterministic policy achieving all $\cardS$ maxima in \eqref{eq: lemma 1 eq 2}. Then we have
\begin{align}
    v & \leq \tilde v = \oonestepr - \oonestepl \ooptimalr + \otransmatrix v \label{eq: lemma 1 eq 2.1}\\
    (I - \otransmatrix) v & \leq \oonestepr - \oonestepl \ooptimalr \label{eq: lemma 1 eq 2.2}.
\end{align}
Multiplying both sides by $\olimitingmatrix \geq 0$, we have 
\begin{align*}
    \olimitingmatrix (I - \otransmatrix)v \leq \olimitingmatrix(\oonestepr - \oonestepl \ooptimalr).
\end{align*}
The l.h.s. is $0$ because $\olimitingmatrix \otransmatrix = \olimitingmatrix$. The r.h.s. $\leq 0$ because
\begin{align*}
    \olimitingmatrix(\oonestepr - \oonestepl \ooptimalr) \leq \olimitingmatrix(\oonestepr - \oonestepl \optionr (\opi)) = 0,
\end{align*}
where we use $\ooptimalr \geq \optionr (\opi), \forall\ \pi \in \osetpolicies$.
Combining the above two inequalities, we have
\begin{align}
    \olimitingmatrix(\oonestepr - \oonestepl \ooptimalr) = 0 \label{eq: lemma 1 eq 3}.
\end{align}

Now we need to use the following lemma, which is essentially the same as Lemma 2.1 by S\&F, except that the signs of inequalities are reversed. The proof follows the same arguments as those used in the proof of Lemma 2.1 by S\&F. 
\begin{lemma}
Fix any policy $\pi \in \osetpolicies$, suppose that $\olimitingmatrix b = 0$ and $(I - \otransmatrix) x - b \leq 0$, then $(I - \olimitingmatrix)x - \ofundmatrix b \leq 0$.
\end{lemma}
Let $x = v$, and $b = \oonestepr - \oonestepl \ooptimalr$, we see that $\olimitingmatrix b = 0$ because of \eqref{eq: lemma 1 eq 3} and $(I - \otransmatrix) x - b \leq 0$ because of \eqref{eq: lemma 1 eq 2.2}. Using the above lemma, we have 
\begin{align*}
    v \leq \ofundmatrix(\oonestepr - \oonestepl \ooptimalr) + \olimitingmatrix v.
\end{align*}
Inserting this inequality to \eqref{eq: lemma 1 eq 2.1}, we have
\begin{align*}
    \tilde v &= [\oonestepr - \oonestepl \ooptimalr + \otransmatrix v]\\
    & \leq (\oonestepr - \oonestepl \ooptimalr + \otransmatrix [\ofundmatrix(\oonestepr - \oonestepl \ooptimalr) + \olimitingmatrix v])\\
    & = [(I + \otransmatrix \ofundmatrix)(\oonestepr - \oonestepl \ooptimalr) + \olimitingmatrix v]\\
    & = [(\ofundmatrix + \olimitingmatrix)(\oonestepr - \oonestepl \ooptimalr) + \olimitingmatrix v] && \text{by Equation 2.2 in S\&F}\\
    & = [\ofundmatrix(\oonestepr - \oonestepl \ooptimalr) + \olimitingmatrix v] && \text{by Theorem 3.1 (a) in S\&F}\\
    & \leq \max_{\opi \in \osetoptimalpolicies}[\ofundmatrix ( \oonestepr - \oonestepl \ooptimalr) + \olimitingmatrix v]\\
    & = v && \text{because $v$ satisfies \eqref{eq: lemma 1 eq}}.
\end{align*}

Combining $\tilde v \geq v$ and $\tilde v \leq v$, we have $\tilde v = v$ and therefore $
v(s) = \tilde v(s) = \max_{o} \sum_{s', r, l} \otrans(s', r, l \mid s, o) (r - l \ooptimalr + v(s')), \quad \forall\ s \in \sspace.$ And thus $v$ is a solution of \eqref{eq: state value optimality equation}.

\subsection{Proof of \cref{lemma: s&f 1978 5.1 d}}\label{proof of lemma: s&f 1978 5.1 d}

Choose $\opi \in \osetpolicies_{**}$. Because $v, v+x$ are both solutions of \eqref{eq: state value optimality equation}, using part (e)(2) of Theorem 3.1 by S\&F, we have, $\forall\ s \in \ooptimalrecurrentstates$,
\begin{align*}
    v(s) &= [\oonestepr - \oonestepl \ooptimalr + P_{\opi} v](s)\\
    [v+x](s) &= [\oonestepr - \oonestepl \ooptimalr + P_{\opi} (v+x)](s).
\end{align*}
Subtracting, we have $x(s) = [P_{\opi} x](s)$. Iteratively applying this equation, we have 
\begin{align*}
    x(s) = [P^\infty_{\opi} x](s)= \langle d^\alpha_{\opi}, x\rangle, \forall\ s \in R_{*\alpha},
\end{align*}
where $d^\alpha_{\opi}$ is the unique stationary distribution of the $\alpha$-th recurrent class of $P_{\opi}$. This proves \eqref{eq: s&f 1978 5.1 1}.

Because $v+x$ is a solution of \eqref{eq: state value optimality equation}, by \cref{lemma: s&f 1978 Theorem 4.1 c}, $\forall\ s \in \sspace$,
\begin{align*}
    [v+x](s) &= \max_{\opi \in \osetoptimalpolicies} [\ofundmatrix(\oonestepr - \oonestepl \ooptimalr) + \olimitingmatrix [v+x]](s)\\
    x(s) & = \max_{\opi \in \osetoptimalpolicies} [\ofundmatrix(\oonestepr - \oonestepl \ooptimalr + \olimitingmatrix v - v) + \olimitingmatrix x](s) && \text{by Equation 2.2 in S\&F}\\
    & = \max_{\opi \in \osetoptimalpolicies} [\ofundmatrix b_{v, \opi} + \olimitingmatrix x](s)\\
    & = \max_{\opi \in \osetoptimalpolicies} [\ofundmatrix b_{v, \opi}](s) + \sum_{\beta = 1}^{\ooptimalnbrecurrentclasses} \left(\sum_{s' \in R_{*\beta}} \olimitingmatrix(s, s') \right) y_\beta .
\end{align*}

Using the above equality and \eqref{eq: s&f 1978 5.1 1}, \eqref{eq: s&f 1978 5.1 2} and \eqref{eq: s&f 1978 5.1 3} hold.

\subsection{Proof of Non-Convexity}\label{sec: non-convexity}
We now show that both $\calV_\infty$ and $\GDiffQsolutionq$ are not necessarily convex. We will show this point by constructing a counter-example, which involves the communicating MDP shown in the left sub-figure of \cref{fig:communicating_example}. The optimal reward rate for the MDP is $0$.

\begin{figure*}[h]
\centering
    \begin{subfigure}{\textwidth}
    \includegraphics[width=\textwidth]{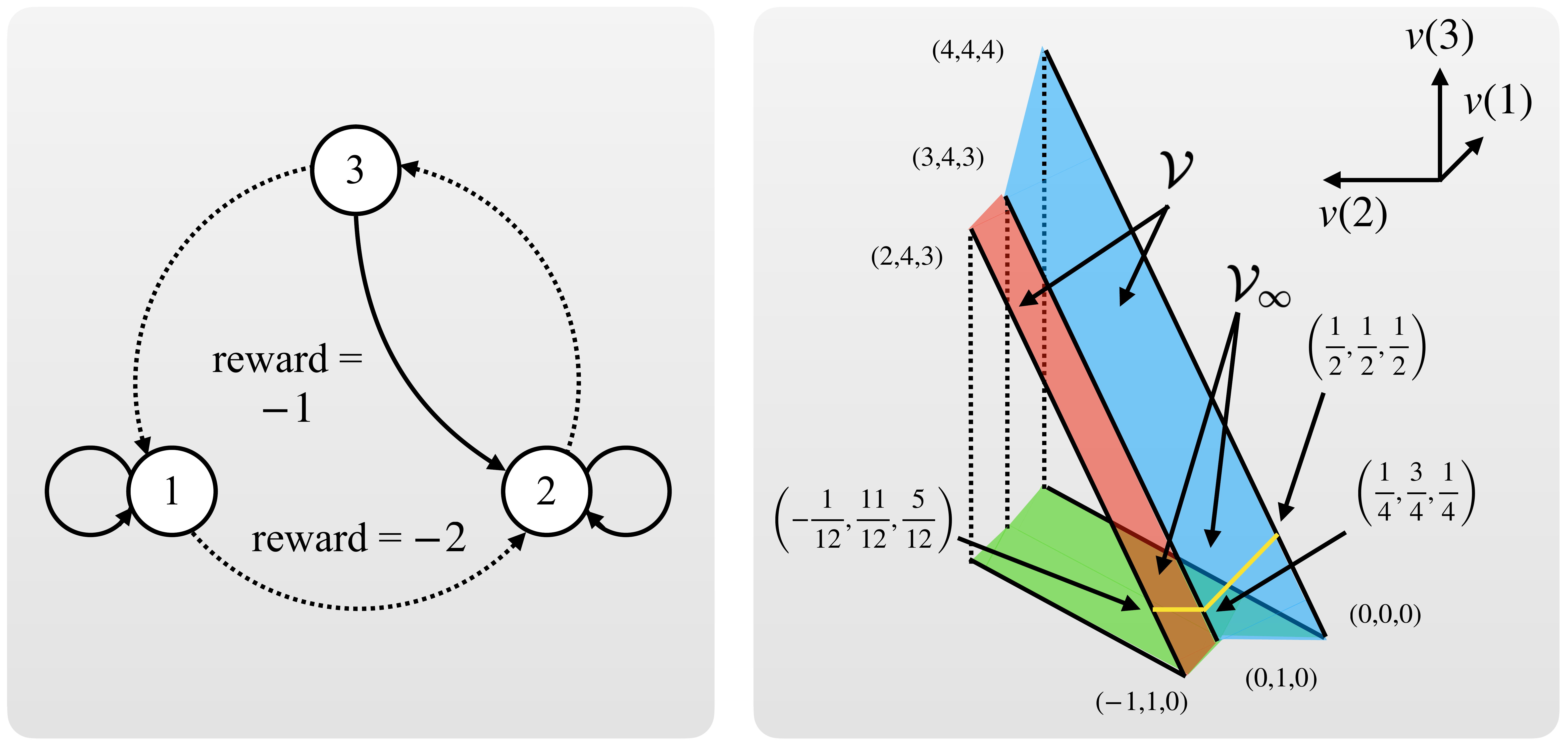}
    \end{subfigure}
    \caption{Illustration example. \emph{Left}: the example MDP. There are three states marked by three circles respectively. There are two actions \texttt{solid} and \texttt{dashed}, both have deterministic effects. Taking action \texttt{solid} at state \emph{3} results in a reward of $-1$. Taking action \texttt{dashed} at state \emph{1} results in a reward of $-2$. All other rewards are $0$.  \emph{Right}: a graphical explanation of $\calV, \calV_\infty$. The two yellow line segments together represent the solution set $\calV_\infty$, which is not convex. The red and blue regions together represent $\calV$.}
    \label{fig:communicating_example}
\end{figure*}


Let $f(q) = \sum_{s, a} q(s, a)$. Such a choice of $f$ satisfies the assumption on $f$ in \cref{thm: characterize Q}. Then by definition, for any $v \in \calV_\infty$, we have
\begin{align*}
    f(\tilde r + \tilde P v) = \ooptimalr.
\end{align*}
For the three-state MDP considered here, $\ooptimalr = 0$.
\begin{align*}
    f(\tilde r + \tilde P v) &= \sum_{s, a} r(s, a) + p(s' \mid s, a) v(s')\\
    &= -3 + 2v(1) + 3v(2) + v(3).
\end{align*}
Therefore,
\begin{align*}
    2v(1) + 3v(2) + v(3) = 3.
\end{align*}
In addition to the above equality, $\calV_\infty$ needs to satisfy the state-value optimality equation (\cref{eq: state value optimality equation}). Therefore for any $v \in \calV_\infty$,
\begin{align*}
    v(1) & = \max(v(1), -2 + v(2)),\\
    v(2) & = \max(v(2), v(3)),\\
    v(3) & = \max(v(1), -1 + v(2)),
\end{align*}
which implies
\begin{align*}
    v(1) & \geq -2 + v(2),\\
    v(2) & \geq v(3),\\
    v(3) & = \max(v(1), -1 + v(2)).
\end{align*}
Therefore 
\begin{align*}
    & \calV_\infty = \{v \in \bbR^3 \mid v(1) \geq -2 + v(2); v(2) \geq v(3); v(3) = \max(v(1), -1 + v(2)); \\
    & 2 v(1) + 3 v(2) + v(3) = 3\}.
\end{align*}
Graphically, $\calV_\infty$ corresponds to the two connected yellow line segments in the right sub-figure in \cref{fig:communicating_example}. From the figure, we see that $\calV_\infty$ is not convex. 

Let $s$ and $d$ denote \texttt{solid} and \texttt{dashed} respectively. Consider any $q \in \calQ_\infty$, in view of \eqref{eq: q solution determination},
\begin{align*}
    q(1, s) + q(1, d) + q(2, s) + q(2, d) + q(3, s) + q(3, d) = 0.
\end{align*}
In addition to the above equality, $\calQ_\infty$ needs to satisfy the action-value optimality equation (\cref{eq: action-value optimality equation}). Therefore for any $q \in \calQ_\infty$,
\begin{align*}
    q(1, s) & = 0 - 0 + \max(q(1, s), q(1, d))\\
    q(1, d) & = -2 - 0 + \max(q(2, s), q(2, d))\\
    q(2, s) & = 0 - 0 + \max(q(2, s), q(2, d))\\
    q(2, d) & = 0 - 0 + \max(q(3, s), q(3, d))\\
    q(3, s) & = -1 - 0 + \max(q(2, s), q(2, d))\\
    q(3, d) & = 0 - 0 + \max(q(1, s), q(1, d)),
\end{align*}
which implies
\begin{align*}
    q(1, s) & \geq q(1, d)\\
    q(1, d) & = -2 + \max(q(2, s), q(2, d))\\
    q(2, s) & \geq q(2, d)\\
    q(2, d) & = \max(q(3, s), q(3, d))\\
    q(3, s) & = -1 + \max(q(2, s), q(2, d))\\
    q(3, d) & = \max(q(1, s), q(1, d))\\
\end{align*}

Consider two solutions $q_1, q_2 \in \calQ_\infty$ defined as follows:
\begin{align*}
    q_1(1, s) &= \frac{1}{2}, q_1(1, d) = -\frac{3}{2},\\ 
    q_1(2, s) &= \frac{1}{2}, q_1(2, d) = \frac{1}{2},\\ 
    q_1(3, s) &= -\frac{1}{2}, q_1(3, d) = \frac{1}{2},\\
    q_2(1, s) &= -\frac{2}{3}, q_2(1, d) = -\frac{2}{3}, \\
    q_2(2, s) &= \frac{4}{3}, q_2(2, d) = \frac{1}{3}, \\
    q_2(3, s) &= \frac{1}{3}, q_2(3, d) = -\frac{2}{3} .
\end{align*}
The midpoint of $q_1$ and $q_2$, $\bar q \doteq 0.5 q_1 + 0.5 q_2$, satisfies
\begin{align*}
    \bar q(1, s) &= -\frac{1}{12} , \bar q(1, d) = -\frac{13}{12}, \\
    \bar q(2, s) &= \frac{11}{12} , \bar q(2, d) = \frac{5}{12}, \\
    \bar q(3, s) &= -\frac{1}{12} , \bar q(3, d) = -\frac{1}{12}.
\end{align*} 
Note that 
\begin{align*}
    \frac{5}{12} = \bar q(2, d) \neq \max(\bar q(3, s), \bar q(3, d)) = -\frac{1}{12}.
\end{align*}
Therefore $\bar q$ does not satisfy the action-value optimality equation (\eqref{eq: action-value optimality equation}) and $\bar q \not \in \GDiffQsolutionq$. Thus $\GDiffQsolutionq$ is not convex.

\subsection{General RVI Q} \label{sec: General RVI Q}

We now start to present the General RVI Q algorithm.

Let $\ispace \doteq \{1, 2, \cdots, k\}$ where $k$ is a positive integer. Consider solving $\bar r \in \bbR$ and $q \in \bbR^{\cardI}$ in following equation
\begin{align} \label{eq: General RVI Q Bellman equation}
    r(i) - \bar r + g(q)(i) - q(i) = 0, \forall\ i \in \ispace
\end{align}
where $r \in \bbR^\cardI$ is any fixed $\cardI$-dim vector, and $g: \bbR^{\cardI} \to \bbR^{\cardI}$ satisfies \cref{assu: g}.

We now consider an algorithm solving \eqref{eq: General RVI Q Bellman equation}. This algorithm maintains an $\cardI$-dim vector of estimates $Q \in \bbR^{\cardI}$, and updates $Q$ using
\begin{align}
    Q_{n+1}(i) &\doteq Q_n(i) + \alpha_{\nu(n, i)} \big( R_n(i) - F_n(Q_n) + G_n(Q_n)(i) - Q_n(i) + \epsilon_n(i) \big) I\{i \in Y_n\} \label{eq: General RVI Q async update},
\end{align}
where 
\begin{enumerate}
    \item $\{Y_n\}$ is the ``update schedule'' -- it is a set-valued process taking values in the set of nonempty subsets of $\ispace$ with the interpretation: $Y_n = \{i: i\textsuperscript{th}$ component of $Q$ was updated at time $n\}$,
    \item $\nu(n, i) \doteq \sum_{k=0}^n \bfI\{i \in Y_k\}$, where $\bfI$ is the indicator function (i.e., $\nu(n, i) =$ the number of times the $i$ component was updated up to step $n$)
    \item $\{\alpha_n\}$ is a step-size sequence,
    \item $\{R_n\}$ is a sequence of i.i.d. random vectors satisfying $\bbE \left [R_n \right] = r, \forall\ n = 0, 1, \dots$,
    \item for any $Q \in \bbR^{\cardI}$, $\{G_n(Q)\}$ is a sequence of i.i.d. random vectors satisfying $ \bbE[G_n(Q)(i)] = g(Q)(i), \forall\ i \in \ispace, n = 0, 1, \dots$, 
    \item for any $Q \in \bbR^{\cardI}$, $\{F_n(Q)\}$ is a sequence of i.i.d. random variables satisfying $ \bbE[F_n(Q)] = f(Q), \forall\ n = 0, 1, \dots$ where $f: \ispace \to \bbR$ is a function satisfying \cref{assu: f},
    \item $\epsilon_n$ is a sequence of random vectors of size $\cardI$. 
\end{enumerate}

We need the following assumptions in addition to Assumptions~\ref{assu: stepsize}--\ref{assu: f}.

\begin{assumption}\label{assu: g}
1) $g$ is a max-norm non-expansion, 2) $g$ is a span-norm non-expansion, 3) $g(x + ce) = g(x) + ce$ for any $c \in \bbR, x \in \bbR^{\cardI}$, 4) $g(cx) = cg(x)$ for any $c \in \bbR, x \in \bbR^{\cardI}$.
\end{assumption}

Let 
\begin{align} \label{eq: GRVIQ Martingale difference sequence}
    M_{n+1} \doteq R_n - r + G_n(Q_n) - g(Q_n) - (F_n(Q_n) - f(Q_n))e.
\end{align}
Let $\calF_n \doteq \sigma(Q_0, \epsilon_0, M_1, \epsilon_1,  M_2, \dots, \epsilon_{n-1}, M_n)$ be the increasing family of $\sigma$-fields. By the above construction, $\{M_{n+1}\}$ is a martingale difference sequence w.r.t. $\calF_n$. That is $\bbE[M_{n+1} \mid \calF_n] = 0$ a.s., $n \geq 0$.

\begin{assumption}\label{assu: variance of Martingale difference}
For $n \in \{0, 1, 2, \dots \}$, $\bbE [\norm{R_{n} - r}^2 \mid \calF_n] \leq K$, $\bbE [\norm{G_{n}(Q) - g(Q)}^2 \mid \calF_n] \leq K(1 + \norm{Q}^2)$ for any $Q \in \bbR^{\abs{\calI}}$, and $\bbE [\norm{F_{n}(Q) - f(Q)e}^2 \mid \calF_n] \leq K(1 + \norm{Q}^2)$ for any $Q \in \bbR^{\abs{\calI}}$ for a suitable constant $K > 0$.
\end{assumption}

We make the following assumption on $\epsilon_n$.
\begin{assumption}\label{assu: epsilon}  
$\bbE[\norm{\epsilon_n}^2 \mid \calF_n] < K(1 + \norm{Q_n}^2)$ a.s.. Further, $\epsilon_n$ converges to 0 a.s..
\end{assumption}






\begin{assumption}\label{assu: unique solution r}
Equation \ref{eq: General RVI Q Bellman equation} has a unique solution of $\bar r$. That is, there exists a pair $\bar r \in \bbR, q \in \bbR^\cardI$ satisfying \eqref{eq: General RVI Q Bellman equation} and if both $\bar r_1, q_1$ and $\bar r_2, q_2$ are solutions \eqref{eq: General RVI Q Bellman equation}, $\bar r_1 = \bar r_2$.
\end{assumption}

Denoted the unique solution of $\bar r$ by $\GRVIQsolutionrbar$. 

Define $\GRVIQsolutionq$ to be the set of $q \in \bbR^{\cardI}$ satisfying \eqref{eq: General RVI Q Bellman equation} and 
\begin{align}
    \GRVIQsolutionrbar & = f(q) \label{eq: General RVI q solution determination}.
\end{align}

\begin{assumption}\label{assu: solution q}
$\GRVIQsolutionq$ is non-empty, bounded, closed, and connected.
\end{assumption}

\begin{assumption}\label{assu: solution 0 reward}
If $r(i) = 0, \forall\ i \in \ispace$, then $0$ is the only element in $\GRVIQsolutionq$.
\end{assumption}

\Cref{assu: option assumption} is assumed in Section 2 of Wan et al. (2021b). Assumptions~\ref{assu: stepsize}--\ref{assu: epsilon} are the same as Assumptions A.1-A.7 by Wan et al. (2021b). Assumptions~\ref{assu: unique solution r}--\ref{assu: solution 0 reward} replace Assumption A.8 by Wan et al. (2021b).

\begin{theorem}\label{thm: General RVI Q}
Under Assumptions~\ref{assu: stepsize}-\ref{assu: solution 0 reward}, General RVI Q (\eqref{eq: General RVI Q async update}) converges, almost surely, $Q_n$ to $\GRVIQsolutionq$ and $f(Q_n)$ to $\GRVIQsolutionrbar$.
\end{theorem}
\begin{proof}
The proof would by a large degree repeats that of Theorem A.1 by Wan et al. (2021). For simplicity, we only highlight modifications.

Both our proof and the proof of Theorem A.1 study two ordinary differential equations (ODEs):
\begin{align}
    \dot y_t & \doteq T_1 (y_t) - y_t , \label{eq: aux ode}\\
    \dot x_t & \doteq T_2 (x_t) - x_t = T_1(x_t) - x_t +  \left ( \GRVIQsolutionrbar - f(x_t) \right)e, \label{eq: original ode}
\end{align}
where
\begin{align*}
    T_1 (Q)(i) & \doteq r(i) + g(Q)(i) - \GRVIQsolutionrbar, \\
    T_2 (Q)(i) & \doteq r(i) + g(Q)(i) - f(Q)\\
    & = T_1 (Q)(i) +  \left (\GRVIQsolutionrbar - f(Q) \right).
\end{align*}



Lemmas A.1 and A.3 by Wan et al. (2021) and their proofs still hold (by replacing $r_\infty$ with $\GRVIQsolutionrbar$).
Lemma A.2 by Wan et al. (2021) is replaced by the following one. The proof follows the same arguments as those for Lemma A.2.
\begin{lemma} \label{lemma: unique equilibrium}
The set of equilibrium points of \eqref{eq: original ode} is $\GRVIQsolutionq$.
\end{lemma}


Lemma A.4 by Wan et al. (2021) is replaced with the following two lemmas.
\begin{lemma} \label{lemma: globally asymptotically stable equilibrium lemma}
If $r(i) = 0, \forall\ i \in \ispace$, $0$ is the globally asymptotically stable equilibrium for \eqref{eq: original ode}.
\end{lemma}

\begin{proof}
We have assumed that $0$ is the only element in $\GRVIQsolutionq$ in Assumption \ref{assu: solution 0 reward} and $0$ is thus the unique equilibrium of \eqref{eq: original ode}. The rest follows the proof of Lemma A.4 by Wan et al. (2021) by replacing $q_\infty$ with $0$.
\end{proof}

\begin{lemma} \label{lemma: Q is internally chain transitive invariant}
$\GRVIQsolutionq$ is a compact connected internally chain transitive invariant set for the ODE \eqref{eq: original ode}. Furthermore, any set that contains points not in $\GRVIQsolutionq$ is not a compact connected internally chain transitive invariant set for the ODE \eqref{eq: original ode}.
\end{lemma}
\begin{proof}
According to Assumption \ref{assu: solution q}, $\GRVIQsolutionq$ is closed and bounded and is thus compact. Assumption \ref{assu: solution q} also assumes that $\GRVIQsolutionq$ is connected. $\GRVIQsolutionq$ is an internally chain transitive invariant set for the ODE \eqref{eq: original ode} because every element in $\GRVIQsolutionq$ is an equilibrium point of the ODE and $\GRVIQsolutionq$ is connected.

If a set contains a point $x \in \bbR^{|\calI|}$ that is not in $\GRVIQsolutionq$, this set can not be internally chain transitive because $x$ is "transient" and the trajectory of \eqref{eq: original ode} can not be arbitrarily close to $x$ at some arbitrarily large time step.
\end{proof}

The last step is to show convergence of a synchronous version of \eqref{eq: General RVI Q async update}
\begin{align}
    Q_{n+1}(i) &\doteq Q_n(i) + \alpha_{\nu(n, i)} \big( R_n(i) - F_n(Q_n) + G_n(Q_n)(i) - Q_n(i) + \epsilon_n(i) \big). \label{eq: General RVI Q sync update}
\end{align}
This result, together with Assumptions \ref{assu: asynchronous stepsize 1} and \ref{assu: asynchronous stepsize 2}, guarantees convergence of the asynchronous update (\eqref{eq: General RVI Q sync update}) by applying results from Section 7.4 by Borkar (2009)

\begin{lemma}\label{lemma: General RVI Q synchronous update convergence}
Equation~\ref{eq: General RVI Q sync update} converges a.s. $Q_{n}$ to $\GRVIQsolutionq$ as $n \to \infty$.
\end{lemma}

\begin{proof}
The proof essentially follows that of Lemma A.5 with two changes. First, using Lemma \ref{lemma: globally asymptotically stable equilibrium lemma} we can show that
the ODE $\dot x_t = h_\infty(x_t) = g (x_t) - f(x_t) e - x_t$ has the origin as the unique globally asymptotically stable equilibrium. Second, the proof of Lemma A.5 uses Borkar’s (2009) Theorem 2, which proves that $Q_n$ converges to a (possibly
sample path dependent) compact connected internally chain transitive invariant set of $\dot x_t = h(x_t)$ where 
\begin{align*}
    h(Q_n)(i) & \doteq r(i) - f(Q_n) + g(Q_n)(i) - Q_n(i).
\end{align*}
Lemma \ref{lemma: Q is internally chain transitive invariant} and Theorem 2 by Borkar (2009) together imply that $Q_n$ must converge to $\GRVIQsolutionq$.
\end{proof}
\end{proof}

\subsection{Verifying \cref{assu: unique solution r}--\cref{assu: solution 0 reward}} \label{sec: Verifying assumptions}
When casting General RVI Q to RVI Q-learning, \eqref{eq: General RVI Q Bellman equation} becomes \eqref{eq: action-value optimality equation}.  It is then clear that \cref{assu: unique solution r} satisfies because the associated MDP is communicating, \cref{assu: solution q} holds because of \cref{thm: characterize Q} and \cref{assu: solution 0 reward} holds because of \cref{lemma: 0 reward MDP has unique solution}.

When casting General RVI Q to Differential Q-learning, \eqref{eq: General RVI Q Bellman equation} becomes \eqref{eq: action-value optimality equation} for an MDP with the same transition dynamics as the original one and all rewards being shifted by a constant that depends on initial action-value estimate $Q_0$ and reward rate estimate $\bar R_0$. It is clear that this new MDP, just like the original one, is communicating. Therefore \cref{thm: characterize Q} and \cref{lemma: 0 reward MDP has unique solution} hold and thus \cref{assu: unique solution r}--\cref{assu: solution 0 reward} hold.

When casting General RVI Q to inter-option Differential Q-learning, \eqref{eq: General RVI Q Bellman equation} becomes \eqref{eq: option-value optimality equation} for an SMDP with the same transition dynamics as the original one and the reward of each state-option pair being shifted in proportional to its expected duration. Again the resulting SMDP is communicating and therefore \cref{thm: characterize Q} and \cref{lemma: 0 reward MDP has unique solution} hold and thus \cref{assu: unique solution r}--\cref{assu: solution 0 reward} hold.

When casting General RVI Q to intra-option Differential Q-learning, \eqref{eq: General RVI Q Bellman equation} becomes \eqref{eq: intra-option-value optimality equation}, which is the same as \eqref{eq: option-value optimality equation} by \cref{prop: inter = intra equations} for the reward shifted SMDP introduced in the previous paragraph. 

\subsection{Convergence of Reward Rates of Greedy Policies}\label{sec: convergence of r(pi_t)}
\begin{lemma}
Assume that the SMDP is weakly communicating, suppose $Q_n$ converges to $\GDiffQsolutionq$ almost surely, let $\opi_n$ be a greedy policy w.r.t. $Q_n$, $r(\opi_n, s) \to \ooptimalr$ almost surely.
\end{lemma}
\begin{proof}
We will need the following lemma for the proof.
\begin{lemma}\label{lemma: smdp lemma r pi_t convergence}
For any $m \in \{1, 2, 3, \ldots\}$,  $a, b \in \bbR^m$, $b > 0$, and for any $p \in \bbR^m$ such that $p \geq 0$, $\sum_s p(s) = 1$,
\begin{align*}
    \frac{p^\top a}{p^\top b}& \geq \min_s \frac{a(s)}{b(s)}, \\
    \frac{p^\top a}{p^\top b}& \leq \max_s \frac{a(s)}{b(s)}.
\end{align*}
\end{lemma}
\begin{proof}
For any $s \in \{1, 2, 3, \ldots, m\}$
\begin{align*}
    \frac{a(s)}{b(s)} &\geq \min_{s'} \frac{a(s')}{b(s')}\\
    a(s)  &\geq b(s) \min_{s'} \frac{a(s')}{b(s')}\\
    p(s) a(s)  &\geq p(s) b(s) \min_{s'} \frac{a(s')}{b(s')}\\
\end{align*}
Therefore 
\begin{align*}
    p^\top a  &\geq p^\top b \min_{s'} \frac{a(s')}{b(s')}\\
\end{align*}
By our assumptions on $b$ and $p$, $p^\top b > 0$, we have 
\begin{align*}
    \frac{p^\top a}{p^\top b}  &\geq  \min_{s'} \frac{a(s')}{b(s')}.
\end{align*}

$\frac{p^\top a}{p^\top b} \leq \max_s \frac{a(s)}{b(s)}$ can be shown in the same way.
\end{proof}

Given that $Q_n$ converges to $\GDiffQsolutionq$, consider $\optionr(\opi_n, s)$ where $\opi_n$ is a greedy policy w.r.t. $Q_n$. We show that $\optionr(\opi_n, s)$ converges to $\ooptimalr$ for all $s \in \sspace$.

For any $\opi \in \osetpolicies$, let $\otransmatrix$ denote the $\cardS \times \cardO \times \cardS \times \cardO$ transition probability matrix under policy $\opi$. That is,
\begin{align}
    \otransmatrix (s, o, s', o') \doteq \sum_{r, l} \otrans (s', r, l \mid s, o) \opi(o' \mid s').
\end{align}

Let $\olimitingmatrix$ be the \emph{limiting matrix} of $\otransmatrix$, which is the Cesaro limit of the sequence $\{\otransmatrix^i\}_{i = 1}^\infty$:
\begin{align*}
    \olimitingmatrix \doteq \lim_{n \to \infty} \frac{1}{n} \sum_{i = 0}^{n-1} \otransmatrix^i.
\end{align*}
Because $\sspace$ is finite, the Cesaro limit exists and $\olimitingmatrix$ is a stochastic matrix (has row sums equal to 1). 

Let $\optionr (s, o) \doteq \sum_{s', r, l} \otrans (s', r, l \mid s, o) r$ be the one-stage option reward and $\optionl (s, o) \doteq \sum_{s', r, l} \otrans (s', r, l \mid s, o) l$ be the one-stage option length. Let $\optionr (\opi, s, o) \doteq \sum_{s', r, l} \otrans(s', r, l \mid s, o) \optionr (\opi, s')$ be the reward rate of policy $\opi$ starting from $s, o$.  By part (a) of Theorem 11.4.1 in Puterman (1994),
\begin{align*}
    \optionr(\opi_n, s, o) = \frac{P_{\opi_n}^\infty \optionr (s, o)}{P_{\opi_n}^\infty \optionl (s, o)}.
\end{align*}
Thus we have,
\begin{align*}
    \optionr(\opi_n, s, o) & = \frac{P_{\opi_n}^\infty \optionr (s, o)}{P_{\opi_n}^\infty \optionl (s, o)}\\
    & = \frac{P_{\opi_n}^\infty (\optionr + P_{\opi_n} Q_n - Q_n) (s, o)}{P_{\opi_n}^\infty \optionl (s, o)} \\
    & \geq \min_{s', o'}  \frac{(\optionr + P_{\opi_n} Q_n - Q_n) (s', o')}{ \optionl (s', o')} \\
    & = \min_{s', o'} \frac{(T Q_n - Q_n)(s', o')}{ \optionl(s', o')},
\end{align*}
where the inequality holds because of \cref{lemma: smdp lemma r pi_t convergence}, and $TQ_n (s, o) \doteq \optionr(s, o) + \sum_{s', r, l} \otrans(s', r, l \mid s, o) \max_{o'} Q_n(s', o')$.

Because the SMDP is weakly communicating, consider a deterministic optimal policy $\opistar \in \osetoptimalpolicies$, $\optionr (\opistar, s) = \ooptimalr, \forall\ s \in \sspace$. Therefore $\optionr (\opistar, s, o) = \ooptimalr$.  Now we have, for any $s, o$,
\begin{align*}
    \ooptimalr & = \optionr(\opistar, s, o) \\
    & = \frac{\olimitingmatrixstar  \optionr (s, o)}{\olimitingmatrixstar  \optionl (s, o)}\\
    & = \frac{\olimitingmatrixstar ( \optionr + P_{\opistar} Q_n - Q_n) (s, o)}{\olimitingmatrixstar  \optionl (s, o)} \\
    & \leq \max_{s', o'} \frac{ ( \optionr + P_{\opistar} Q_n - Q_n) (s', o')}{  \optionl (s', o')} \\
    & \leq \max_{s', o'} \frac{ ( \optionr + P_{\opi_n} Q_n - Q_n) (s', o')}{ \optionl (s', o')}\\
    & = \max_{s', o'} \frac{(T Q_n - Q_n)(s', o')}{ \optionl(s', o')}.
\end{align*}
The first inequality holds because of \cref{lemma: smdp lemma r pi_t convergence} and the second inequality holds because $\opi_n$ is a greedy policy w.r.t. $Q_n$.

With the above results, and that $\ooptimalr \geq \optionr(\opi_n, s, o)$, we have, for any $s, o$,
\begin{align*}
    \min_{s', o'} \frac{(T Q_n - Q_n)(s', o')}{ \optionl(s', o')} \leq \optionr(\opi_n, s, o) \leq \ooptimalr \leq \max_{s', o'} \frac{(T Q_n - Q_n)(s', o')}{ \optionl(s', o')}
\end{align*}

Therefore,
\begin{align*}
    & \max_{s, o} \ooptimalr - \optionr(\opi_n, s, o) \leq sp\left (\frac{T Q_n - Q_n}{ \optionl} \right),
\end{align*}
where $sp(x) = \max_i x(i) - \min x(i)$ denotes the span of vector $x$.

Because $Q_n \to \GDiffQsolutionq$ a.s., every point $q$ in $\GDiffQsolutionq$ satisfies $sp((T q - q)/\optionl) = 0$ because $(T q - q)/\optionl = \ooptimalr e$ by \eqref{eq: option-value optimality equation}. In addition, $sp((TQ_n - Q_n)/\optionl)$ is a continuous function of $Q_n$, by continuous mapping theorem, $sp((TQ_n - Q_n)/\optionl) \to 0$ a.s.. Therefore we conclude that $\optionr(\opi_n, s, o) \to \ooptimalr, \forall s, o$. By definition, $\optionr(\opi_n, s) = \sum_{o} \opi_n(o \mid s) \optionr(\opi_n, s, o)$. Therefore $\optionr(\opi_n, s) \to \ooptimalr$
\end{proof}

\section{Empirical Results} \label{sec: empirical validation}
In this section, we empirically verify our convergence results of Differential Q-learning and RVI Q-learning by showing the dynamics of estimated action values of the two algorithms in a communicating MDP and a weakly communicating MDP.

We first consider the communicating MDP shown at the bottom of \cref{fig:different mdps}. For this MDP, we apply Differential Q-learning with initial action values $0$, initial reward rate estimate $-3$, and $\eta = 1$. 
The behavior policy chooses action \texttt{solid} with probability $0.8$ and action \texttt{dashed} with probability $0.2$ for both two states. The stepsize is $0.1$. We performed 10 runs for each algorithm. Each run starts from state 1 and lasts for $1000$ steps. Every 10 steps, we recorded the estimated action values and plotted the higher action value for each state in the figure. \Cref{fig: empirical validation communicating}(left) shows the evolution of these action values. In the right panel of the same figure, we also show it using RVI Q-learning with action values being initialized with $0$ and $f(q) = q(1, \texttt{dashed})$. A more detailed explanation is provided in the figure's caption. It can be seen that for both algorithms, 1) for each run, the estimated action-value function converged to a point in the solution set (the black line segments), and 2) for different runs, the estimated action values generally converged to different points in the solution set.

\begin{wrapfigure}{r}{0.5\textwidth}
    \centering
    \includegraphics[width=0.5\textwidth]{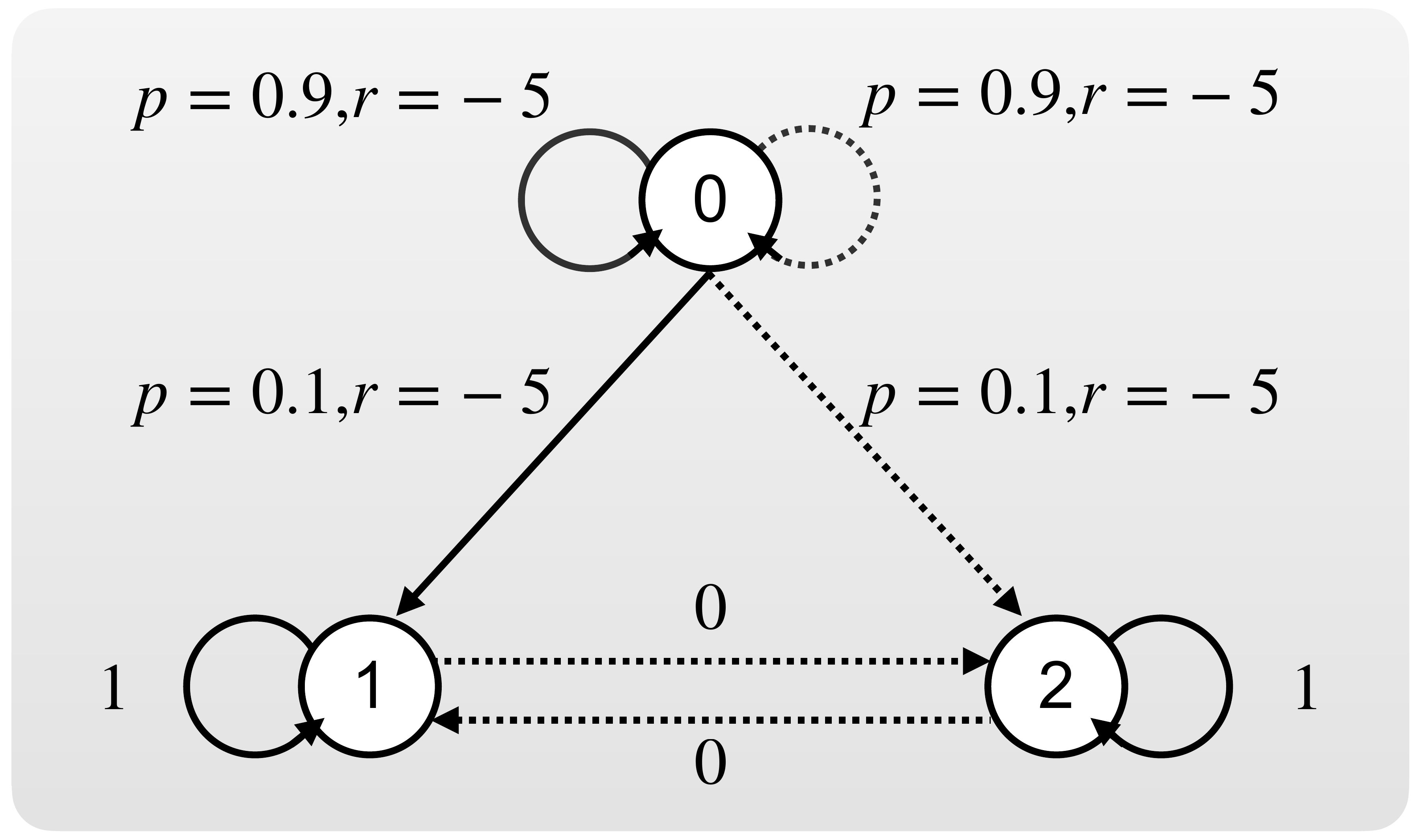}
\caption{A weakly communicating MDP modified from the communicating MDP shown at the bottom of \cref{fig:different mdps} by adding a transient state $0$.}
\label{fig: wc mdp experiment}
\end{wrapfigure}

\begin{figure*}[h]
\centering
    \begin{subfigure}{\textwidth}
    \includegraphics[width=\textwidth]{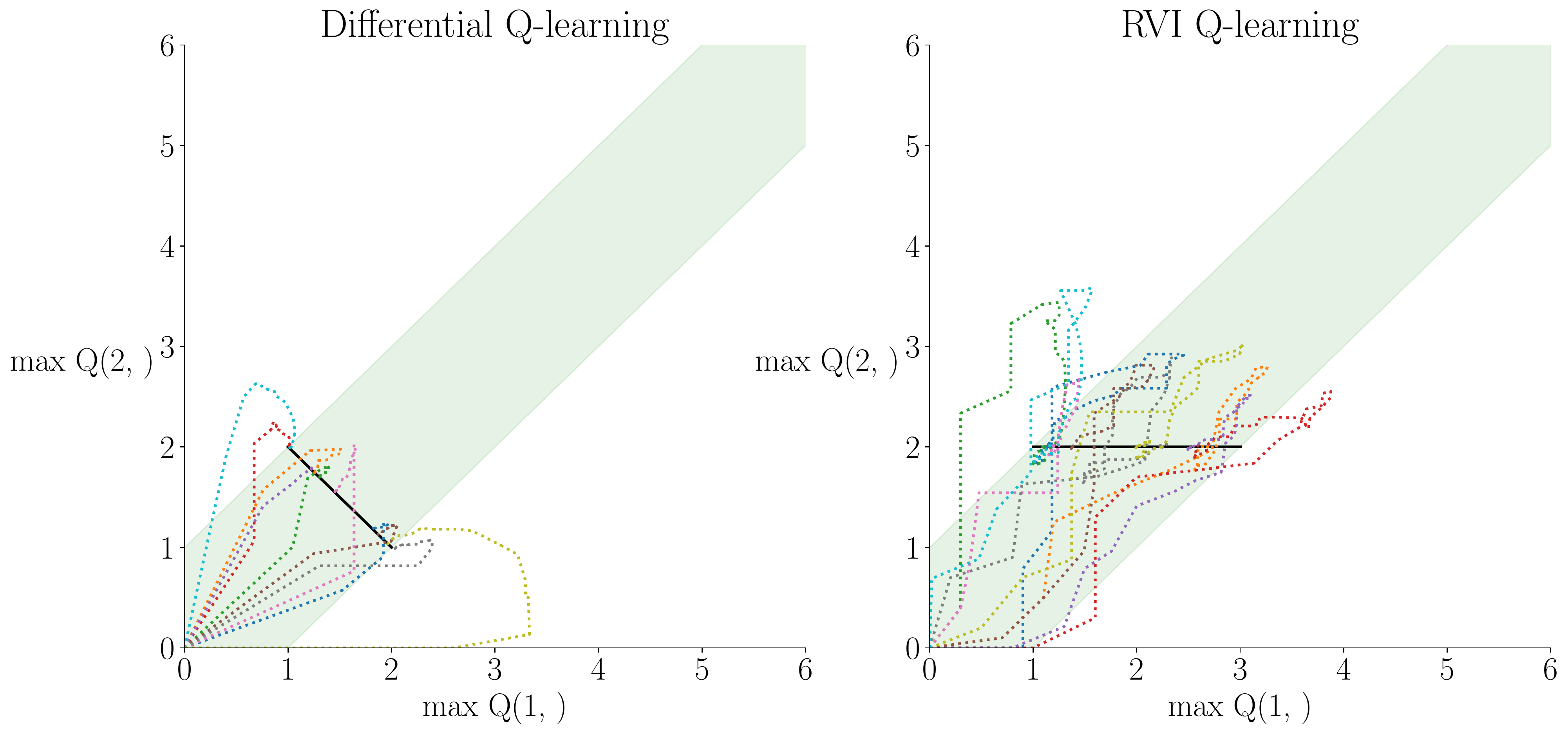}
    \end{subfigure}
    \caption{Evolution of estimated action values of Differential Q-learning and RVI Q-learning in the communicating MDP shown at the bottom of \cref{fig:different mdps}. In each figure, the $x$ and $y$ axes are the higher estimated action value at state $1$ and $2$, respectively. The light green region marks the solution set of the action-value optimality equation (\cref{eq: action-value optimality equation}), and the black line segment marks the solution set $\GDiffQsolutionq$. Each colored dotted trajectory marks the evolution of the estimated action values. Each trajectory starts from zero point and ends at some point on the black line segment.}
    \label{fig: empirical validation communicating}
\end{figure*}

We also applied both of the two algorithms, with the same parameter settings and initialization in a weakly communicating MDP (\cref{fig: wc mdp experiment}), which is just the communicating MDP plus a transient state. In the transient state, taking both \texttt{solid} and \texttt{dashed} actions stays at the transient state with probability $0.9$. The MDP moves to state 1 with probability $0.1$ given action \texttt{solid} and to state 2 with probability $0.1$ given action \texttt{dashed}. The reward starting from state 0 is always $-5$. The starting state is $0$. Because the agent could spend different amounts of time in the transient state for different runs, the agent may enter the communicating set, which contains states 1 and 2, with different action values associated with state $0$. 

The solution set of Differential Q-learning depends on the action values associated with the transient states when entering the communicating class. Therefore in the figure, the points that the estimated action-value function converged to, corresponding to different runs, are not in a line. Nevertheless, the estimated action-value function in all runs converged to the green region, which corresponds to the solution set of the action-value optimality equation.
\begin{figure*}[h]
\centering
    \begin{subfigure}{\textwidth}
    \includegraphics[width=\textwidth]{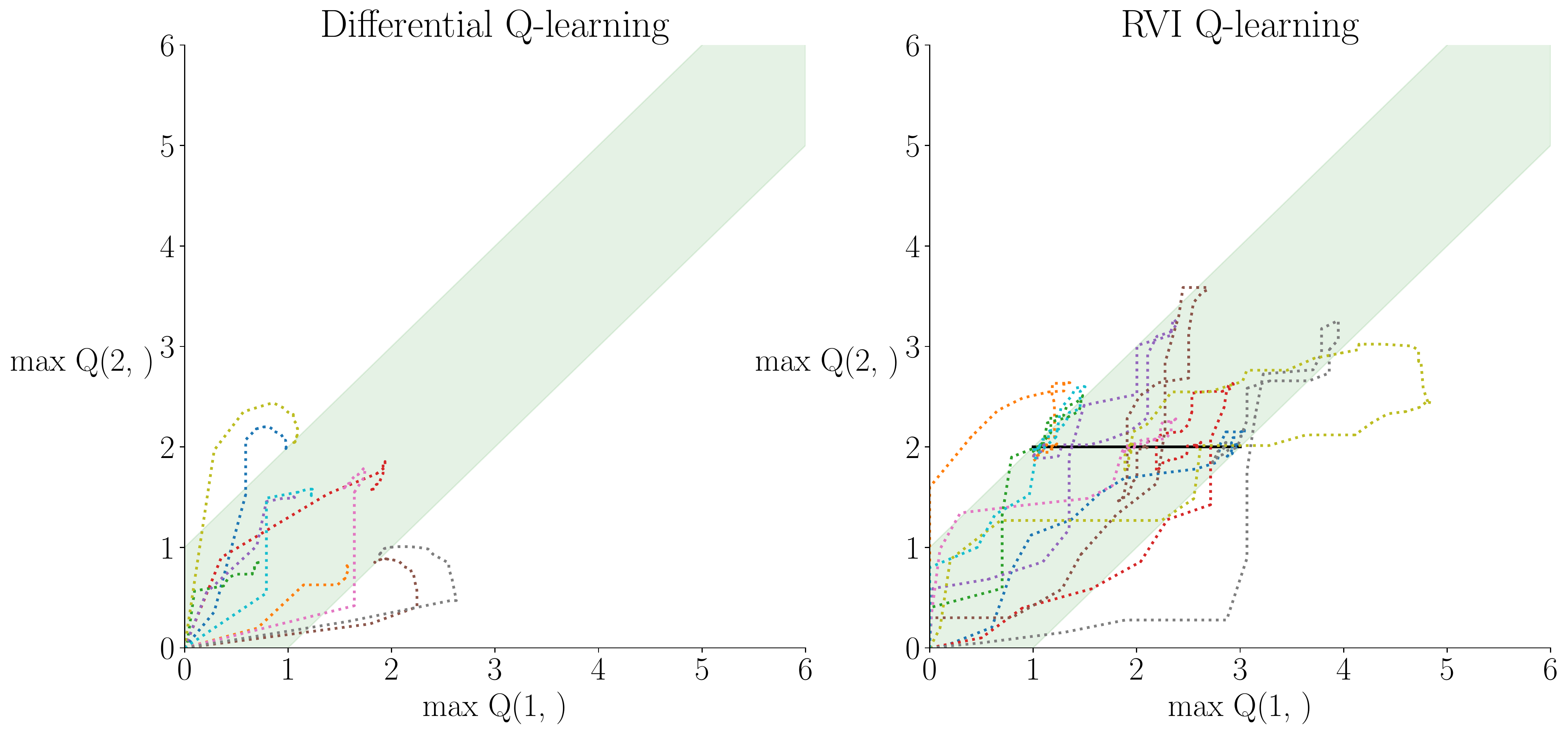}
    \end{subfigure}
    \caption{Evolution of estimated action values of Differential Q-learning and RVI Q-learning in the weakly communicating MDP shown in \cref{fig: wc mdp experiment}.}
    \label{fig: empirical validation weakly communicating}
\end{figure*}

On the other hand, The solution set of RVI Q-learning with the choice of the reference function $f(q) = q(1, \texttt{dashed})$ does not depend on the action values associated with states in the communicating class when entering the class. Therefore the solution set did not vary across different runs. Note that if we chose $f(q) = q(0, \texttt{dashed})$, then again the solution set of RVI Q-learning has that dependence.

\section{Gosavi's (2004) Convergence Result Is Incorrect}\label{sec: Gosavi incorrect proof}
The convergence result of Gosavi's proposed algorithm is presented in Theorem 2 of his paper. In the proof of the theorem, they used Borkar's two-time scale stochastic approximation result to prove the convergence of the proposed algorithm. Specifically, they argued that their algorithm is a special case of the general class of algorithms considered in Borkar's result. As Gosavi quotes, "Note that the Eqs. (48) and (49) for SMDPs form a special case of the general class of algorithms (29) and (30) analyzed using the lemma given in Section 5.1.1. " However, a closer look at these equations shows that equation (49) is not a special case of equation (30). Note that because $\rho^k$ is a scalar, $y^k$ only has one element and thus the $f$ function in equation (30) does not vary across different state-option pairs. However, this is not true for the $f$ function in equation (49). It appears to us that there is no simple fix for this issue.
\end{document}